%% file: main.tex
\documentclass{article}
\usepackage[numbers,  sort&compress]{natbib}
\usepackage[preprint]{neurips_2026}

\usepackage[utf8]{inputenc}  
\usepackage[T1]{fontenc}        
\usepackage{hyperref}              
\usepackage{url}                        
\usepackage{booktabs}              
\usepackage{amsfonts}              
\usepackage{nicefrac}              
\usepackage{xcolor}        
\usepackage{amsmath,amssymb}
\usepackage{microtype}            
\usepackage{amsthm}  
\usepackage{bbm}
\usepackage{cleveref}
\usepackage{algorithm}
\usepackage{graphicx}
\usepackage{enumitem}
\usepackage{makecell} 
\usepackage{multirow}
\usepackage{nicematrix}
\usepackage{pifont}   

\newcommand{\cmark}{\ding{51}}
\newcommand{\xmark}{\ding{55}}

\usepackage{algpseudocode}
\usepackage{xcolor}    
\usepackage{tcolorbox}
\setlist[itemize]{leftmargin=0.5cm}

\newtheorem{theorem}{Theorem}[section]
\newtheorem{proposition}[theorem]{Proposition}
\newtheorem{lemma}[theorem]{Lemma}
\newtheorem{corollary}[theorem]{Corollary}
\newtheorem{definition}[theorem]{Definition}

\theoremstyle{remark}

\newcommand{\std}[1]{{\tiny $\pm\,#1$}} 

\title{Hitting Time Isomorphism for Multi-Stage Planning with Foundation Policies}

\author{%
  \makebox[\textwidth][c]{%
    \textbf{Magnus Victor Boock} \quad \textbf{Abdullah Akgül} \quad \textbf{Mustafa Mert \c{C}elikok} \quad \textbf{Melih Kandemir}
  } \\
  \noalign{\vskip 0.7ex}
  \makebox[\textwidth][c]{Department of Mathematics and Computer Science} \\
  \makebox[\textwidth][c]{University of Southern Denmark} \\
}

\begin{document}

\maketitle

\begin{abstract}
We present a new operator-theoretic representation learning framework for offline reinforcement learning that recovers the directed temporal geometry of a controlled Markov process from hitting time observations. While prior art often produces symmetric distances or fails to satisfy the triangle inequality, our framework learns a Hilbert-space displacement geometry where expected hitting times are realized as linear functionals of latent displacements. We prove that this representation exists under latent linear closure and is uniquely identifiable up to a bounded linear isomorphism. For finite-dimensional implementations, we show that global hitting-time error is bounded by one-step transition error amplified by the environment's transient spectral radius. Furthermore, we provide finite-sample guarantees accounting for approximation, statistical complexity, and trajectory-label mismatch.  Derived from this theory, we curate Isomorphic Embedding Learning (IEL) as a new goal-agnostic foundation policy learning algorithm that  anchors a HILP-style consistency objective with explicit hitting-time regression to ensure that the learned geometry reflects actual decision-time progress. This asymmetric and compositional structure enables robust graph-based multi-stage planning for long-horizon navigation. Our experiments  demonstrate that IEL improves the state of the art of learning foundation policy policies from offline maze locomotion data. Our code can be found at \url{https://github.com/MagnusBoock/IEL/}.
\end{abstract}

\section{Introduction}

The increasing availability of large-scale offline trajectory data has shifted reinforcement learning toward reward-free pretraining: extracting from a data set a reusable policy that generalizes across an open set of goals~\citep{laskin2021urlb,park2024hilp,touati2023does} in a zero-shot manner. This objective, often described as \emph{foundation policy learning}~\citep{park2024hilp}, has motivated a range of large-scale efforts, from unsupervised RL pretraining to Vision--Language--Action (VLA) models that train a single policy on heterogeneous trajectory data and deploy it across many previously unseen tasks~\citep{kim2024openvla,black2024pi_0, reed2022a,rt1,rt2}. We hypothesize that for such a policy to succeed at long-horizon tasks, the offline learned goal-agnostic representation must support the three capabilities jointly. Firstly, it should comply with \emph{unsupervised policy training}, where the policy is learned without exposure to explicit goal-conditioned rewards, ensuring generalization to tasks defined only at deployment. Secondly, the state representations should be queryable with \emph{directional} distances, as most real environments are irreversible or reversible at asymmetric costs, making it possible that routes planned with symmetric distances pass through subgoals unreachable in the required direction. Thirdly, these distances should satisfy the \emph{triangle inequality} so that during multi-stage planing, local distances can compose into globally consistent plans  through intermediate states never specified at training time.

Existing methods follow a trichotomy based on the quantity of interest used for offline state representation learning: successor measures, reachability distances, and hitting times. Successor measure methods \citep{touati2021learning,touati2023does,agarwal2025proto, myers2024cmd} decouple environment dynamics from rewards to allow zero-shot task adaptation. However, these methods—except CMD—fail to satisfy the triangle inequality, making them unable to compose distances across the intermediate sub-goals essential for multi-stage planning. Reachability distance methods \citep{wang2023qrl,liu2023mrn,myers2025tmd}  impose metric or quasimetric structures to improve compositional reasoning. Their key limitation is that they rely on goal-conditioned supervision for policy training. Hitting time methods \citep{park2024hilp,baek2025gas} are uniquely tailored for unsupervised policy training and satisfy the triangle inequality. However, their symmetric Euclidean features fail to capture the irreversible directed geometry of real-world physical transitions, limiting their stitching (hence multi-stage planning) capabilities. Furthermore, GAS \citep{baek2025gas} also necessitates real goal state labels for policy training, hence it does not qualify as a foundation policy learning method. Table~\ref{tab:methods-comparison} gives an overview of the prior art with respect to their properties relevant for foundation policy learning.

\begin{table}[t!]
\centering
\caption{Classification of offline goal-conditional RL methods based on properties essential for foundation policies. Unsupervised policy training denotes the ability to learn a representation serving an open set of goals without retrieving explicit goal-conditioned reward labels from the dataset. No prior method satisfies the triangle inequality, directional distance, and unsupervised training simultaneously. HILP \citep{park2024hilp} is the only existing baseline qualifying for the latter. See Section \ref{app:goal_agnostic_policy_training} of the Appendix for a detailed discussion about policy supervision in prior art. We propose IEL as the first approach to satisfy all three. CMP stands for Controlled Markov Process, i.e. MDP without reward.}
\label{tab:methods-comparison}
\renewcommand{\arraystretch}{1.25}
\resizebox{\textwidth}{!}{%
\begin{tabular}{@{}lcccccc@{}}
\toprule
\textbf{Method} & \makecell[b]{\textbf{Information} \\ \textbf{Distilled}} & \multicolumn{3}{c}{\textbf{Distance}} & \makecell[b]{\textbf{Temporal}\\{\bf Difference}} & \makecell[b]{\textbf{Unsupervised} \\ \textbf{Policy} \\ {\bf Training}} \\ \cmidrule(lr){3-5}
 & & \textbf{Quantity} & \textbf{Directional} & \makecell[b]{\textbf{Triangle} \\ \textbf{Inequality}} & & \\ \midrule
\multicolumn{7}{@{}l}{\textit{Successor measure methods}} \\ \midrule
FB \citep{touati2021learning,touati2023does} & \makecell[t]{Successor  measure} & Bilinear projection & \cmark & \xmark  & \cmark & \xmark \\
C-Learning \citep{eysenbach2021clearning} & \makecell[t]{Reachability} & \makecell[t]{Classifier logits} & \cmark & \xmark &  \cmark & \xmark \\
CRL \citep{eysenbach2022crl} & Density ratio & Inner product & \cmark & \xmark  & \xmark & \xmark \\
PSM \citep{agarwal2025proto} & \makecell[t]{Successor measure} &Bilinear projection & \cmark & \xmark & \cmark & \xmark \\
CMD \citep{myers2024cmd} & \makecell[t]{Successor features} & \makecell[t]{Log-density ratio} & \cmark & \cmark &  \xmark & \xmark \\ \midrule
\multicolumn{7}{@{}l}{\textit{Reachability distance methods}} \\ \midrule
HIQL \citep{park2023hiql} & \makecell[t]{Multi-step reachability} & Value gap & \cmark & \xmark  & \cmark & \xmark \\
QRL \citep{wang2023qrl} & \makecell[t]{Q-values} & \makecell[t]{Implicit Q-distance} & \cmark & \cmark & \xmark & \xmark \\
TMD \citep{myers2025tmd} & \makecell[t]{Time-to-arrival} & \makecell[t]{Quasimetric  feature} & \cmark & \cmark &  \cmark & \xmark \\ 
\midrule
\multicolumn{7}{@{}l}{\textit{Hitting-time  methods}} \\ \midrule
GAS \citep{baek2025gas} & \makecell[t]{Hitting time} & Euclidean feature & \xmark & \cmark  & \cmark & \xmark \\
HILP \citep{park2024hilp} & \makecell[t]{Hitting time} & Euclidean feature & \xmark & \cmark &  \cmark & \cmark \\
\textbf{IEL (This work)} & \makecell[t]{\textbf{Hitting time}} & \makecell[t]{\textbf{Isomorphism to CMP}} & \textbf{\cmark} & \textbf{\cmark}  & \textbf{\cmark} & \textbf{\cmark} \\ \bottomrule
\end{tabular}%
}
\end{table}

We identify the missing primitive of the prior approaches as the \emph{directional} temporal geometry measured in discrete transition steps. We propose \emph{expected hitting time} as this primitive. For a policy $\pi$ and goal $g$, $V^\pi(x,g)$ represents the expected number of steps to reach $g$ from $x$. This quantity is characterized by a Poisson equation where a constant cost of one is incurred at each step until reaching the goal, at which point the cost becomes zero. This formulation provides the exact object required for multi-stage planning, bypassing the need to recover the temporal structure from surrogates such as successor measures or reachability distances.  We formalize this idea via an operator-theoretic framework. For processes with a Hilbert-space feature map $\phi$ and approximately linear latent dynamics, the hitting time satisfies the identity $V^\pi(x,g) = \langle \phi(g)-\phi(x),\, \omega_g^\pi\rangle_{\mathcal H}$, where $\omega_g^\pi \in \mathcal{H}$ identifies the task of navigating to state $g$. This proves that representations capturing hitting-time functionals are equivalent up to a bounded linear isomorphism, implying our learning target is the invariant directed temporal geometry of the environment rather than an arbitrary coordinate system. Theoretically, we establish that this displacement-linear representation exists under latent linear closure and an adjoint solvability condition (Theorem~\ref{thm:existence_of_linear_hitting_time_representation}) and is uniquely identifiable under a separating reference policy (Theorem~\ref{thm:identifiability_of_linear_hitting_time_representation}). For finite-dimensional implementations, we prove that global hitting-time error is bounded by one-step latent transition error amplified by a transient-spectral-radius horizon factor (Theorem ~\ref{thm:approximation_error}). Finally, our finite-sample regression analysis (Theorem~\ref{thm:pac_oracle_displacement_topology}) decomposes total error into approximation, statistical-complexity, coverage, and trajectory-label-mismatch terms, providing a complete PAC-oracle guarantee for the resulting displacement topology.

We use our theoretical framework to develop \emph{Isomorphic Embedding Learning (IEL)} as the first foundation policy training algorithm tailored specifically for multi-stage planning. IEL trains $\phi$ from offline trajectories by relabeling intermediate states as goals and regressing observed temporal gaps onto displacement-projected scores anchored by a learned task identifier $\omega(g)$. At test time, IEL constructs a graph over offline latent states with edge costs given by the learned hitting-time score, and applies shortest-path search to convert local temporal predictions into global zero-shot navigation.  As shown in Table~\ref{tab:methods-comparison}, IEL is the only framework that integrates unsupervised policy training with a directional distance and the triangle inequality.  We demonstrate the necessity of these qualities for learning foundation policies with multi-stage planning in experiments conducted on six offline goal-conditioned RL benchmarks. IEL's directional distance features and the graph-based planning mechanism bring significant and complementary performance gains.
 
\section{Isomorphic Embedding Learning from Hitting Time Observations}

\subsection{Problem statement}
We consider a dynamical environment formulated as a \emph{Controlled Markov Process (CMP)} defined by the tuple $(\mathcal{X}, \mathcal{A}, \mathcal{P})$, where $\mathcal{X}$ and $\mathcal{A}$ are real-valued, potentially multidimensional state and action spaces. The transition dynamics are governed by a Borel measurable kernel $\mathcal{P}: \mathcal{X} \times \mathcal{A} \to \Delta(\mathcal{X})$, which assigns a probability measure over next states $x'$ given a current state-action pair $(x, a) \in \mathcal{X} \times \mathcal{A}$. We assume the existence of a fixed \emph{behavior policy} $\pi_B: \mathcal{X} \to \Delta(\mathcal{A})$ used to provide us with an offline dataset $D = \{\tau_i\}_{i=1}^n$ consisting of $n$ independent and identically distributed (i.i.d.) trajectories $\tau_i = (x_0, a_0, x_1, a_1, \dots, x_{T_i})$ of length $T_i$ generated by the interaction of $\pi_B$ with the kernel $\mathcal{P}$.

Let $T_{x \rightarrow g}^\pi = \inf \{t \geq 0 \mid x_0=x, x_t = g\}$ be a random variable representing the \emph{hitting time} to a goal state $g \in \mathcal{X}$ under a policy $\pi$. We assume that for any $g$ reachable from $x$, there exists at least one policy in the admissible set $\Pi_D$ such that the expected hitting time $V^\pi(x, g) = \mathbb{E}[T_{x \rightarrow g}^\pi \mid x_0 = x]$ is finite. Our objective is to identify the \emph{optimal in-distribution foundation policy} $\pi_{D,g}^*$ defined as
$\pi_{D,g}^* = \arg\min_{\pi \in \Pi_D} V^\pi(x, g)$ for all $g \in \mathcal{X}$
where the admissible policy set $\Pi_D$ is defined by the \emph{in-distribution criterion} of \emph{absolute continuity}. Formally, for a fixed state $x \in \mathcal{X}$, let $\pi(\cdot|x)$ and $\pi_B(\cdot|x)$ be the probability measures over the Borel $\sigma$-algebra $\mathcal{B}(\mathcal{A})$. We define the in-distribution criterion of absolute continuity as
$\Pi_D =  \{\pi: \mathcal{X} \to \Delta(\mathcal{A}) \mid \forall x \in \mathcal{X}, \pi(\cdot|x) \ll \pi_B(\cdot|x)  \}.$
The notation $\pi \ll \pi_B$ signifies that $\pi(\cdot|x)$ is absolutely continuous with respect to $\pi_B(\cdot|x)$, which is symbolically defined as:
$\forall B \in \mathcal{B}(\mathcal{A})$ with $\pi_B(B|x) = 0$ implying $\pi(B|x) = 0.$
By the \emph{Radon-Nikodym Theorem}, this condition is necessary and sufficient for the existence of a non-negative and measurable function $w: \mathcal{X} \times \mathcal{A} \to [0, \infty)$, names as the Radon-Nikodym derivative, such that for any $B \in \mathcal{B}(\mathcal{A})$ we have $\pi(B|x) = \int_B w(a|x) \, d\pi_B(a|x).$
This criterion ensures that the optimal policy only assigns mass to actions within the support of the behavior policy.  We treat the trajectories in the data set as the empirical basis for estimating the expectations of the hitting time variable and the underlying transition dynamics. We refer to $\pi_{D,g}^*$ as a \emph{foundation policy} due to its capability to solve a large set of policy search tasks in a zero-shot manner.

We develop our solution using a Hilbert space framework because it combines high expressive power with geometric consistency. Hilbert spaces are a standard tool in control theory for the analysis of nonlinear dynamical systems, particularly through the use of Koopman and Perron-Frobenius operators \citep{bouvrie2017kernel, Klus2017EigendecompositionsOT, brunton2022modern, kostic2024learning}. This framework has recently inspired new approaches in offline reinforcement learning \citep{da2026latent, touati2021learning, touati2023does} and the development of foundation policies \citep{park2024hilp} by providing a structured way to represent state-action transitions as linear operators in a latent space. We approach the foundation policy learning problem from a rigorous Hilbert space perspective. The definition below specifies the conditions under which a CMP is representable within a Hilbert space subject to a bounded approximation error $\epsilon$. 

\begin{definition}[$\epsilon$-Sufficient CMP]
\label{def:cmp}
A CMP $(\mathcal{X}, \mathcal{A}, \mathcal{P})$ is said to be $\epsilon$-Sufficient with respect to a latent representation $(\mathcal{H}, \phi, \mathcal{K})$ if there exists a Hilbert space $\mathcal{H}$, a feature map $\phi: \mathcal{X} \to \mathcal{H}$, and a family of action-indexed bounded linear operators $\{\mathcal{K}_a \in \mathcal{L}(\mathcal{H}) | a \in \mathcal{A}\}$ such that for all $(x, a) \in \mathcal{X} \times \mathcal{A}$:
$\left\| \mathbb{E}_{x' \sim \mathcal{P}(\cdot|x, a)}[\phi(x')] - \mathcal{K}_a \phi(x) \right\|_{\mathcal{H}} \leq \epsilon$
for a given $\epsilon \geq 0$ where $\mathcal{L}(\mathcal{H}) \triangleq \{ \mathcal{K}:  \mathcal{H} \to \mathcal{H} ~~| ~ \|\mathcal{K}\| < \infty \}$ is the Banach space of linear maps with finite operator norm.
\end{definition}

A \emph{Sufficient CMP} is simply an $\epsilon$-Sufficient CMP where $\epsilon = 0$, representing the case where the transition dynamics are perfectly captured by the linear action-operators in the latent Hilbert space. This definition is the Hilbert-space analogue of the linear-feature transition condition underlying linear action models \citep{yao2012approximate, lehnert2020successor}, with $\epsilon$ playing the role of the misspecification term familiar from approximate linear MDPs \citep{zanette2020learning}. It is also closely related to conditional mean feature models of MDP transition dynamics \citep{song2009hilbert, grunewalder2012modelling, lever2016compressed}, which embed the full conditional $\mathcal{P}(\cdot | x, a)$ in an RKHS rather than matching expected features under a fixed $\phi$.

\subsection{The Methodology: Hitting Time Isomorphism}

Our main methodological contribution builds on the key fact that in a sufficient CMP, it is actually possible to relate the expected hitting time to feature difference. This fact bridges the gap between temporal dynamics and the Hilbert space geometry, using the property that the features $\phi$ act as a system of coordinates for the latent state space. The below result states that if the feature map $\phi$ is rich enough to linearize the dynamics (via the transition operators $\{\mathcal{K}_a\}$) and expressive enough to represent the passage of time (via the range condition on the adjoint), then the \emph{expected hitting time is exactly a linear functional of the displacement in the feature space}. Throughout, we work on the closed feature span
$\overline{\operatorname{span}}\{\phi(x):x\in\mathcal X\}$,
so equality of inner products against all feature vectors identifies elements of $\mathcal H$. For hitting-time identities with goal $g$, the relevant feature span is the closed span generated by $\{\phi(x):x\neq g\}$.

\begin{theorem}[Existence of linear hitting time representation in CMPs]
\label{thm:existence_of_linear_hitting_time_representation}
Let $(\mathcal{X}, \mathcal{A}, \mathcal{P})$ be a Sufficient CMP with latent representation $(\mathcal{H}, \phi, \mathcal{K})$. Assume that $
\sup_{x \in \mathcal{X}} \|\phi(x)\|_{\mathcal H} < \infty.$ For a fixed policy $\pi(a|x)$, let $\mathcal{K}_\pi \in \mathcal{L}(\mathcal{H})$ be the {\bf policy-integrated transition operator} such that for all $x \in \mathcal{X}$: $
\mathcal{K}_\pi \phi(x) = \mathbb{E}_{a \sim \pi(\cdot|x)} [ \mathcal{K}_a \phi(x) ],$
and let $\mathbf{w}_{\mathbf{1}} \in \mathcal{H}$ be the representer of the unit cost functional. The expected hitting time $V^\pi(x, g)$ admits the linear representation $V^\pi(x, g) = \langle \phi(g) - \phi(x), \omega_g^\pi \rangle_{\mathcal{H}}$
for some vector $\omega_g^\pi \in \mathcal{H}$ if and only if $\mathbf{w}_{\mathbf{1}} \in \text{range}(\mathcal{K}_\pi - I)^*$ on the goal-$g$ transient feature span.
\end{theorem}

For any policy $\pi$ such that the policy-induced latent dynamics close linearly on $\mathcal{H}$, i.e., there exists a bounded linear operator $\mathcal{K}_\pi \in \mathcal{L}(\mathcal{H})$ satisfying $
\mathcal{K}_\pi \phi(x)=\mathbb{E}_{a\sim \pi(\cdot\mid x)}[\mathcal{K}_a \phi(x)]$ for all $x\in \mathcal{X}$,
Theorem~\ref{thm:existence_of_linear_hitting_time_representation} shows that the expected hitting time admits a linear representation in the displacement coordinates whenever the unit-cost functional lies in the range  of $(\mathcal{K}_\pi-I)^*$. Thus, once the range condition holds, the passage of time is encoded by the latent linear dynamics rather than by an ad hoc distance heuristic. The next result establishes the stronger identifiability statement ---namely that alternative representations recovering the same hitting-time geometry must be related by a bounded linear change of coordinates---under additional assumptions. Specifically, we assume the existence of a separating policy that can reach any goal and that the mapping from latent displacements to hitting times is well-conditioned. These requirements are reasonable as they merely ensure the environment is fully explorable and that the feature space does not merge distinct temporal transitions.

\begin{theorem}[Identifiability of Displacement Geometry]
\label{thm:identifiability_of_linear_hitting_time_representation}
Let $(\mathcal{X},\mathcal{A},\mathcal{P})$ be a Sufficient CMP with canonical representation $(\mathcal{H}^\star,\phi^\star,\mathcal{K}^\star)$ and
alternative representation $(\mathcal{H},\phi,\mathcal{K})$. Fix a reference point $x_0\in \mathcal{X}$ and define
\[
\mathcal D:=\overline{\mathrm{span}}\{\phi(g)-\phi(x_0):g\in \mathcal{X}\},\quad
\mathcal D^\star:=\overline{\mathrm{span}}\{\phi^\star(g)-\phi^\star(x_0):g\in \mathcal{X}\}.
\]
Assume there exists at least one admissible policy $\pi_{\mathrm{sep}} \in \Pi_D$ such that: (i) Both representations satisfy the hitting-time representation of Theorem~2.2 under $\pi_{\mathrm{sep}}$:
\[
V^{\pi_{\mathrm{sep}}}(x,g)=\langle \phi(g)-\phi(x), \omega_g^{\pi_{\mathrm{sep}}}\rangle_\mathcal{H}=\langle \phi^\star(g)-\phi^\star(x), \omega_g^{\star,\pi_{\mathrm{sep}}}\rangle_{\mathcal{H}^\star},\quad \forall x,g\in \mathcal{X};
\] (ii) for a common full-support probability measure $\nu$ on $\mathcal{X}$, the analysis operators
$(\mathcal A_{\pi_{\mathrm{sep}}} d)(h):=\langle d,\omega_h^{\pi_{\mathrm{sep}}}\rangle_\mathcal{H}$ and $(\mathcal B_{\pi_{\mathrm{sep}}} d^\star)(h):=\langle d^\star,\omega_h^{\star,\pi_{\mathrm{sep}}}\rangle_{\mathcal{H}^\star}$
define bounded, bounded-below maps $\mathcal A_{\pi_{\mathrm{sep}}}:\mathcal D\to L^2(\mathcal{X},\nu)$ and
$\mathcal B_{\pi_{\mathrm{sep}}}:\mathcal D^\star \to L^2(\mathcal{X},\nu)$.
Then there exists a unique bounded linear isomorphism $M:\mathcal D\to\mathcal D^\star$ such that
for all $x,g\in \mathcal{X}$, $M\bigl(\phi(g)-\phi(x)\bigr)=\phi^\star(g)-\phi^\star(x).$
Additionally, for every admissible policy $\pi' \in \Pi_D$ for which the operators $\mathcal{K}_{\pi'}$ and $\mathcal{K}^\star_{\pi'}$ are well-defined, 
\[
M\bigl(\phi(g)-\mathcal{K}_{\pi'}\phi(x)\bigr)=\phi^\star(g)-\mathcal{K}_{\pi'}^\star\phi^\star(x),\quad \forall x,g\in \mathcal{X},
\]
and for every action $a \in \mathcal{A},$ $M\bigl(\phi(g)-\mathcal{K}_a\phi(x)\bigr)=\phi^\star(g)-\mathcal{K}_a^\star\phi^\star(x).$

\end{theorem}

This identifiability ensures the learned displacement space is a linear transformation of the environment's true geometry, preserving the relative temporal distances between states. Because the isomorphism $M$ connects these representations, any hitting-time information encoded in the canonical dynamics can be pulled back and read out from our learned features as formalized below.

\begin{corollary}[Pullback of Canonical Hitting-time Readouts]
\label{cor:hitting_time_equivalence}
Let $M:\mathcal D\to\mathcal D^\star$ be the bounded linear isomorphism given by
Theorem~\ref{thm:identifiability_of_linear_hitting_time_representation}.
Fix any policy $\pi$ for which the canonical hitting-time representation is available, i.e.,
for each goal $g\in \mathcal X$ there exists a canonical representer $\omega_g^{\star ,\pi}\in \mathcal{H}^\star$ such that $V^\pi(x,g)=\langle \phi^\star(g)-\phi^\star(x),\, \omega_g^{\star,\pi}\rangle_{\mathcal{H}^\star}$ for all $x,g\in \mathcal X$. Then for every goal $g\in \mathcal X$, there exists a unique vector $\bar \omega_g^\pi\in \mathcal D$ such that $\langle \phi^\star(g)-\phi^\star(x),\, \omega_g^{\star,\pi}\rangle_{\mathcal{H}^\star}
=
\langle \phi(g)-\phi(x),\, \bar \omega_g^\pi\rangle_\mathcal{H}$ for all
$x\in \mathcal X$. Consequently, $
V^\pi(x,g)=\langle \phi(g)-\phi(x),\, \bar \omega_g^\pi\rangle_\mathcal{H}, \forall x,g\in \mathcal X.$ Moreover, for every admissible policy $\pi\in\Pi_D$ for which $\mathcal{K}_\pi$ and $\mathcal{K}_\pi^\star$ are well-defined, it holds  for all $x,g\in \mathcal X$ that $\langle \phi^\star(g)-\mathcal{K}_\pi^\star\phi^\star(x),\, \omega_g^{\star,\pi}\rangle_{\mathcal{H}^\star}
=
\langle \phi(g)-\mathcal{K}_\pi\phi(x),\, \bar \omega_g^\pi\rangle_\mathcal{H}.$
\end{corollary}
HILP \cite{park2024hilp} and GAS \cite{baek2025gas} rely on a symmetric $L_2$ distance heuristic that fails in the realistic case where the environment dynamics are directed. Contrarily, our framework identifies an isomorphism between the MDP dynamics and a directed displacement geometry, which captures both the direction and the scale of time. Instead of assuming that features can be forcibly mapped to a spatial distance, our Sufficient CMP framework ensures that the latent coordinates $\phi(x)$ are mathematically constrained by the linear transition operator $\mathcal{K}_a$. This provides a provable link between latent dynamics and hitting times, enabling reliable multi-stage planning through directed graph search—a capability fundamentally missing in symmetric feature methods. Furthermore, we establish an algebraic identifiability absent in the quasimetric representations of TMD \citep{myers2025tmd}, which exploits asymmetric distances to approximate hitting times. This approach lacks the operator-theoretic grounding necessary to ensure that these distances are globally consistent across the state space. In our framework, Theorem ~\ref{thm:existence_of_linear_hitting_time_representation} proves that the expected hitting time is a linear function of latent displacement coordinates whenever the unit-cost functional satisfies the adjoint range condition. Theorem~\ref{thm:identifiability_of_linear_hitting_time_representation} provides the stronger structural guarantee: any representation satisfying this identity is related to the canonical one by a unique bounded linear isomorphism. This fact ensures that our learned geometry is a coordinate-consistent reflection of the underlying process rather than an unconstrained distance heuristic. Consequently, our framework recovers the environment's latent displacement structure with a level of mathematical rigor that current distance-based models do not guarantee.

\input{algorithm_section}

\section{Theoretical Analysis}

This two-part analysis bounds errors from $d$-dimensional latent compression and provides finite-sample PAC guarantees. The results establish conditions for recovering a near-oracle displacement topology—modulo a bounded linear map—while quantifying the impact of approximation, coverage, trajectory-label mismatch, and finite-sample effects.

\subsection{Consistency Analysis}

For the consistency analysis, let $\widetilde{\mathcal H}$ denote an ambient Hilbert space with feature map $\widetilde{\phi}:\mathcal X\to \widetilde{\mathcal H}$. Fix a $d$-dimensional subspace $\mathcal H\subseteq \widetilde{\mathcal H}$ with orthogonal projection $P_d:\widetilde{\mathcal H}\to \mathcal H$, and define the encoder $\phi(x):=P_d\widetilde{\phi}(x)\in\mathcal H.$ We work henceforth in the compressed space $\mathcal H$, so $\dim(\mathcal H)=d$.

\begin{definition}[Sufficient Capacity]
Let $(\mathcal X,\mathcal A,\mathcal P)$ be a CMP. For a fixed policy $\pi$, let $\widetilde{\mathcal T}^\pi$ denote the policy-induced feature transition operator on the ambient feature span, defined by $\widetilde{\mathcal T}^\pi \widetilde{\phi}(x) = \mathbb E_{x'\sim \mathcal P^\pi(\cdot\mid x)}[\widetilde{\phi}(x')].$ We say that the compressed representation $(\mathcal H,\phi)$ has \emph{Sufficient Capacity} $\epsilon$ under $\pi$ if there exists a linear operator $\mathcal T^\pi\in \mathcal L(\mathcal H)$ such that $\|P_d \widetilde{\mathcal T}^\pi \widetilde{\phi}(x)-\mathcal T^\pi \phi(x)\|_{\mathcal H}\le \epsilon$ for all $x\in\mathcal X$.
If the subspace $\mathcal H$ is fixed, we may take $\mathcal T^\pi:=P_d\widetilde{\mathcal T}^\pi\big|_{\mathcal H}$, the compression of the ambient feature transition operator to $\mathcal H$.
\end{definition}

The following lemma shows that the compressed operator $\mathcal T^\pi$ correctly approximates the mean transition of the compressed features.

\begin{lemma}[Consistency of Sufficient Capacity]
\label{lem:consistency_of_sufficient_capacity}
If the representation satisfies the Sufficient Capacity condition, then $\|\mathbb E_{x'\sim\mathcal P^\pi(\cdot\mid x)}[\phi(x')]-\mathcal T^\pi\phi(x)\|_{\mathcal H}\le \epsilon,
\qquad \forall x\in\mathcal X$.
\end{lemma}
Define the out-of-subspace leakage by $\eta(x):=\|(I-P_d)\mathbb E_{x'\sim\mathcal P^\pi(\cdot\mid x)}[\widetilde{\phi}(x')]\|_{\widetilde{\mathcal H}}.$ Then it holds that
 $\|\mathbb E_{x'\sim\mathcal P^\pi(\cdot\mid x)}[\widetilde{\phi}(x')]-\mathcal T^\pi\phi(x)\|_{\widetilde{\mathcal H}}
\le
\epsilon+\eta(x)$. If the ambient features are already subspace-contained, i.e. $\mathcal R(\widetilde{\phi})\subseteq \mathcal H$, then $\eta(x)=0$. 
We now partition the latent space into the goal subspace $\mathcal{H}_g:=\operatorname{span}\{\phi(g)\}$ and the transient subspace $\mathcal{H}_Q:=\mathcal{H}_g^\perp$, assume that $\mathbf{w}_{\mathbf 1}\in\mathcal{H}_Q$, and write $P_Q$ for the orthogonal projection onto $\mathcal{H}_Q$. For a fixed policy $\pi$, let $P^\pi$ denote the policy-induced Markov operator on scalar functions, $(P^\pi f)(x):=\mathbb E_{x'\sim\mathcal P^\pi(\cdot\mid x)}[f(x')]$,  and let $P_Q^\pi$ denote its restriction to the transient state space $\mathcal X\setminus\{g\}$, which is defined as $(P_Q^\pi f)(x):= \mathbb E_{x'\sim\mathcal P^\pi(\cdot\mid x)}
\bigl[
f(x')\mathbf 1\{x'\neq g\}
\bigr]$. Equivalently, $P_Q^\pi f$ is obtained by extending $f$ to the goal with boundary value $f(g)=0$ and then applying $P^\pi$ on transient states. We also define the compressed transient latent operator $\mathcal T_Q^\pi:=P_Q\mathcal T^\pi\big|_{\mathcal H_Q}$. The following theorem
bounds the global error of the hitting-time predictor induced by a representation satisfying the sufficient capacity condition.
\begin{theorem}
\label{thm:approximation_error}
Fix a CMP $(\mathcal X,\mathcal A,\mathcal P)$, a policy $\pi$, and an absorbing goal $g\in\mathcal X$. Suppose that the compressed representation $(\mathcal H,\phi)$ has sufficient capacity $\epsilon$ under $\pi$. Assume the transient unit cost is represented by some $\mathbf w_{\mathbf 1}\in\mathcal H_Q$, i.e. $\langle \phi(x),\mathbf w_{\mathbf 1}\rangle_{\mathcal H}=1$ for all $x\in\mathcal X\setminus\{g\}$, that the goal subspace is invariant under $\mathcal T^\pi$, i.e. $\mathcal T^\pi(\mathcal H_g)\subseteq\mathcal H_g$, and let $\omega_g^\pi\in\mathcal H_Q$ satisfy $(\mathcal T_Q^\pi-I)^* \omega_g^\pi=\mathbf w_{\mathbf 1}$.
If $\bigl\|(I-P_Q^\pi)^{-1}\bigr\|_\infty
\le
C_H/(1-\rho(\mathcal T_Q^\pi))$
for some $C_H\ge 1$, then the representer
$\widehat V_g^\pi(x):=\langle \phi(g)-\phi(x),\,\omega_g^\pi\rangle_{\mathcal H}$ satisfies
\[
\sup_{x\in\mathcal X\setminus\{g\}}
\bigl|V^\pi(x,g)-\widehat V_g^\pi(x)\bigr|
\le
\frac{C_H\,\|\omega_g^\pi\|_{\mathcal H}\,\epsilon}{1-\rho(\mathcal T_Q^\pi)}.
\]
\end{theorem}

This result shows that the approximation error is amplified by the factor $(1-\rho(\mathcal T_Q^\pi))^{-1}$, which acts as an effective-horizon term for the transient latent dynamics. Practically, this means that small one-step representation errors can accumulate more severely when the goal is difficult to reach. Thus, any admissible reference representation whose compressed latent dynamics are sufficiently accurate under the policy of interest induces a hitting-time predictor with controlled global error. We therefore compare the learner to the best such reference geometry available within this class. Concretely, fix $x_0\in \mathcal X$, a policy $\pi_{\mathrm{sep}}\in\Pi_D$, and a full-support
probability measure $\nu$ on $\mathcal X$. Let $\mathfrak R_\epsilon$ denote the class of
admissible reference tuples $r=(\mathcal H^r,\phi^r,\mathcal T^{r,\pi_{\mathrm{sep}}},
\{\omega_g^{r,\pi_{\mathrm{sep}}}\}_{g\in \mathcal X})$ for which Theorem~\ref{thm:approximation_error} applies under $\pi_{\mathrm{sep}}$
with constants $C_H(r,g)$, and for which the associated
analysis operator $(B_r d)(h):=\langle d,\omega_h^{r,\pi_{\mathrm{sep}}}\rangle_{\mathcal H^r}$
is bounded and bounded below, where $d\in \mathcal D_r:=\overline{\mathrm{span}}\{\phi^r(g)-\phi^r(x_0):g\in \mathcal X\}.$
Writing the supremum constants $
C_H(r):=\sup_{g\in \mathcal X}C_H(r,g)$, 
$C_\omega(r):=\sup_{g\in \mathcal X}\|\omega_g^{r,\pi_{\mathrm{sep}}}\|_{\mathcal H^r},$
and $\rho_r:=\sup_{g\in \mathcal X}\rho(\mathcal T_{Q,g}^{r,\pi_{\mathrm{sep}}}),$
we define the best achievable benchmark scale as $\eta_\epsilon^\star
:=
\inf_{r\in\mathfrak R_\epsilon}
(C_H(r)\,C_\omega(r)\epsilon)/(1-\rho_r)$.
Thus $\eta_\epsilon^\star$ is the smallest error level guaranteed by
Theorem~\ref{thm:approximation_error} over all admissible $\epsilon$-Sufficient
compressed representations. Let $\mathcal F$ be the class of predictors realizable by
an algorithm in the form $f(x,g)=\langle \phi(g)-\phi(x),\omega_g\rangle_\mathcal{H}$, and
suppose that this class can match the oracle benchmark up to excess
$\eta_{\mathrm{cls}}$, i.e. $\inf_{f\in\mathcal F}\|f-V^{\pi_{\mathrm{sep}}}\|_{L^2(\nu\times\nu)}
\le
\eta_\epsilon^\star+\eta_{\mathrm{cls}}$.
The next result shows that any population minimizer of the regression objective aligns with the corresponding reference displacement geometry through a bounded linear map.

\begin{theorem}
\label{thm:oracle_displacement_recovery}
Let $\widehat V\in\mathcal F$ with
$\widehat V(x,g)=\langle \phi(g)-\phi(x),\bar \omega_g\rangle_{\mathcal H}$
be a population minimizer of the regression objective
$\mathcal L(f):=\|f-V^{\pi_{\mathrm{sep}}}\|_{L^2(\nu\times\nu)}^2.$
Define $\mathcal D:=\overline{\mathrm{span}}\{\phi(g)-\phi(x_0):g\in \mathcal X\},$
and assume the learned analysis operator $(A d)(h):=\langle d,\bar \omega_h\rangle_{\mathcal H}$ is bounded and bounded-below. Then for every $\delta>0$ and every $r_\delta\in\mathfrak R_\epsilon$
satisfying $C_H(r_\delta)\,C_\omega(r_\delta)\epsilon/(1-\rho_{r_\delta})
\le
\eta_\epsilon^\star+\delta$,
there exists a bounded linear map $M_\delta:D\to D_{r_\delta}$ such that
$\left(
\int_{\mathcal X}\int_{\mathcal X}
\|M_\delta(\phi(g)-\phi(x))-(\phi^{r_\delta}(g)-\phi^{r_\delta}(x))\|_{\mathcal H^{r_\delta}}^2
\,d\nu(x)\,d\nu(g)
\right)^{1/2}
\le
\frac{2}{m_{B,\delta}}
\bigl(2\eta_\epsilon^\star+\eta_{\mathrm{cls}}+\delta\bigr)$,
where $m_{B,\delta}>0$ is the lower-bound constant of $B_{r_\delta}$.
\end{theorem}

\subsection{Finite Sample Analysis}
We analyze the sample complexity of the hitting time regression part of Algorithm~\ref{alg:IEL}, conditioned on a fixed task encoder $\bar \omega_\psi: \mathcal X \to \mathcal H$. Theorem~\ref{thm:pac_oracle_displacement_topology} bounds finite-sample recovery of the temporal displacement geometry associated with the policy defining the regression target. For offline control, the relevant policy is the in-distribution optimum $\pi^*_{D,g}\in\Pi_D$, but trajectories from this policy are not observed. The IQL-style expectile objective therefore provides an offline approximation to its value geometry by emphasizing shorter behavior-supported hitting-time outcomes. The gap to the ideal in-distribution-optimal geometry can thus be viewed as the representation error controlled by Theorem~\ref{thm:pac_oracle_displacement_topology} plus a
separate IQL approximation error. Let $\Phi$ be the class of admissible state encoders, and define the induced predictor class for hitting time regression as
$\mathcal F_{\Phi,\bar  \omega_\psi} :=
\left\{
f_\phi(s,u,g)=\langle \phi(u)-\phi(s),\bar  \omega_\psi(g)\rangle_\mathcal{H}
:
\phi\in\Phi
\right\}.$ Assume $\|\phi(x)\|_\mathcal{H}\le C_\phi$ for all  $\phi\in\Phi,\ x\in \mathcal X$ and
$\|\bar  \omega_\psi(g)\|_\mathcal{H}\le C_{\bar  \omega_\psi}$ for all  $g\in \mathcal X$.
From $N$ i.i.d. trajectories sampled under $\pi_B$, we extract exactly $m_{\mathrm{tr}}$ tuples of the form
$z_{j,t}=(s_{j,t},u_{j,t},g_{j,t},H_{j,t})$ per trajectory, where
$t \in \{1,\dots,m_{\mathrm{tr}} \}$, $u_{j,t}$ is an intermediate state on the same trajectory as
$s_{j,t}$ and $g_{j,t}$, and $H_{j,t}$ is the replay-buffer temporal-gap label extracted for that tuple
(e.g. the number of steps from $s_{j,t}$ to $u_{j,t}$ along the sampled trajectory segment).
Assume $0\le H_{j,t}\le H_{\max}$ almost surely. Let $\mu$ denote the law of a uniformly selected
extracted tuple and
write $(S,U,\mathsf G,H)\sim\mu$ for a generic such tuple. Let $\lambda$ be the $(S,U)$-marginal of $\mu$, and define the ideal displacement target $G^{\pi_B}(s,u,g):=V^{\pi_B}(s,g)-V^{\pi_B}(u,g).$ We also define the replay-buffer conditional mean $m_{\mathrm{traj}}(s,u,g):=\mathbb E[H\mid S=s,U=u,\mathsf G=g].$
Let $\widehat \phi_N\in\arg\min_{\phi\in\Phi}\widehat R_N(\phi)$ be the empirical minimizer of the
tuple-level regression objective, let $\widehat G_N(s,u,g):=\langle \widehat\phi_N(u)-\widehat\phi_N(s),\bar \omega_\psi(g)\rangle_\mathcal{H},$
and define the learned displacement space $
D_N:=\overline{\mathrm{span}}\{\widehat\phi_N(u)-\widehat\phi_N(s):(s,u)\in\operatorname{supp}(\lambda)\}.$
For each extraction slot $t$, define the slotwise trajectory class
$
(\mathcal F_{\Phi,\bar  \omega_\psi})^{(t)}
:=
\left\{
\tau\mapsto f_\phi(s_t(\tau),u_t(\tau),g_t(\tau))
:
\phi\in\Phi
\right\},$
and let $
\overline{\mathfrak R}_N(\Phi,\bar  \omega_\psi)
:=
\frac{1}{m_{\mathrm{tr}}}\sum_{t=1}^{m_{\mathrm{tr}}}
\mathfrak R_N\!\bigl((\mathcal F_{\Phi,\bar  \omega_\psi})^{(t)}\bigr)$
denote the average slotwise trajectory-level Rademacher complexity. The resulting statistical error
term is 
\begin{align*}
\Gamma_N(\delta)
:=
16(2C_\phi C_{\bar  \omega_\psi}+H_{\max})\,\overline{\mathfrak R}_N(\Phi,\bar  \omega_\psi)
+
2(2C_\phi C_{\bar  \omega_\psi}+H_{\max})^2
\sqrt{\frac{\log(1/\delta)}{2N}}.  
\end{align*}
The next result lifts this statistical error into a finite-sample displacement-topology recovery guarantee.
\begin{theorem}
\label{thm:pac_oracle_displacement_topology}
Under the finite-sample tuple-regression setup, assume
$\pi_{\mathrm{sep}}=\pi_B$ and:
(i) for $\lambda$-a.e. $(s,u)$, the conditional goal law $\nu_{s,u}(\cdot):=\mathcal L(\mathsf G\mid S=s,U=u)$ dominates a common full-support probability measure $\nu$ with coverage constant $c_{\mathrm{cov}}\in(0,1]$;
 (ii) the analysis operator $(A_N d)(g):=\langle d,\bar \omega_\psi(g)\rangle_\mathcal{H}$ is bounded as a map $D_N\to L^2(\mathcal X,\nu)$; (iii) the class is oracle-competitive with the ideal displacement target: $ \inf_{\phi\in\Phi}
\|f_\phi-G^{\pi_B}\|_{L^2(\mu)}
\le
2\eta_\epsilon^\star+\eta_{\mathrm{cls}};$ (iv) the replay-buffer temporal-gap labels are $\eta_{\mathrm{traj}}$-compatible with the ideal
displacement geometry: $\|m_{\mathrm{traj}}-G^{\pi_B}\|_{L^2(\mu)}
\le
\eta_{\mathrm{traj}}.$
Then for every $\delta_{\mathrm{or}}>0$ and every reference
$r_{\delta_{\mathrm{or}}}\in\mathfrak R_\epsilon$ satisfying
$\frac{
C_H(r_{\delta_{\mathrm{or}}})\,C_\omega(r_{\delta_{\mathrm{or}}})\,\epsilon
}{
1-\rho_{r_{\delta_{\mathrm{or}}}}
}
\le
\eta_\epsilon^\star+\delta_{\mathrm{or}}$,
there exists a bounded linear map $ T_{N,\delta_{\mathrm{or}}}:D_N\to D_{r_{\delta_{\mathrm{or}}}}$
such that, with probability at least $1-\delta$,
\begin{align*}
\int_{\mathcal X\times \mathcal X}
\bigl\|
T_{N,\delta_{\mathrm{or}}}(\widehat\phi_N(u)&-\widehat\phi_N(s)) 
-
(\phi^{r_{\delta_{\mathrm{or}}}}(u)-\phi^{r_{\delta_{\mathrm{or}}}}(s))
\bigr\|_{\mathcal H^{r_{\delta_{\mathrm{or}}}}}^{2}
\,d\lambda(s,u)    \\
&\le
\frac{4}{c_{\mathrm{cov}}\,m_{B,\delta_{\mathrm{or}}}^{2}}
\Big[
(2\eta_\epsilon^\star+\eta_{\mathrm{cls}}+\eta_{\mathrm{traj}})^2
+
\Gamma_N(\delta)
+
\eta_{\mathrm{traj}}^{2}
+
2(\eta_\epsilon^\star+\delta_{\mathrm{or}})^2
\Big],
\end{align*}
where $m_{B,\delta_{\mathrm{or}}}>0$ is the lower-bound constant of the reference analysis operator
$B_{r_{\delta_{\mathrm{or}}}}$.
\end{theorem}
The term $\eta_{\mathrm{traj}}$ quantifies the discrepancy between replay-buffer temporal gaps and ideal hitting-time targets, such that smaller values indicate better alignment between the learned displacement topology and the oracle reference topology within the established error bounds.


\input{experiments}


\input{conclusion}

\bibliographystyle{plainnat}
\bibliography{refs}


\newpage
\appendix

\section{Basic Linear Analysis Concepts}\label{sec:basic_linear_analysis_concepts}

\begin{definition}[Adjoint]
        Let  $\mathcal{A}:  \mathcal{H}  \to  \mathcal{H}$  be  a  bounded  linear  operator  on  a  Hilbert  space  $\mathcal{H}$.  The  {\bf  adjoint}  of  $\mathcal{A}$,  denoted  by  $\mathcal{A}^*$,  is  the  unique  linear  operator  that  satisfies:  $\langle  \mathcal{A}x,  y  \rangle_{\mathcal{H}}  =  \langle  x,  \mathcal{A}^*y  \rangle_{\mathcal{H}}$  for  all  vectors  $x,  y  \in  \mathcal{H}$.
\end{definition}
\begin{definition}[Linear  operator]
A  linear  operator  $L:  \mathcal{H}  \to  \mathcal{H}^\star$  is  {\bf  bounded}  if  there  exists  some  finite  constant  $M  \geq  0$  such  that  for  all  vectors  $v  \in  \mathcal{H}$:  $\|L  v\|_{\mathcal{H}^\star}  \leq  M  \|v\|_{\mathcal{H}}$.  
\end{definition}
\begin{definition}[Operator norm]
\begin{align*}
    ||T|| &\triangleq \inf \{ M \geq 0: ||Tv|| \leq M ||v||, \forall v \in \mathcal{H} \} =\sup_{||v||=1} ||Tv|| = \sup_{v \neq 0} \frac{||Tv||}{||v||}
\end{align*}
\end{definition}
   

\begin{definition}[Linear isomorphism of displacement spaces]
Let $\phi:\mathcal{X} \to \mathcal{H}$ and $\phi^\star :\mathcal{X} \to \mathcal{H}^\star$ be two feature maps into Hilbert spaces, and fix
$x_0\in \mathcal{X}$. Define the associated displacement spaces
\[
D:=\overline{\mathrm{span}}\{\phi(g)-\phi(x_0):g\in \mathcal{X}\},
\qquad
D^\star:=\overline{\mathrm{span}}\{\phi^\star(g)-\phi^\star(x_0):g\in \mathcal{X}\}.
\]
We say that the two representations are linearly isomorphic on displacement spaces if there exists
a bounded linear operator $M:D\to D^\star$ with bounded inverse such that
\[
M\bigl(\phi(g)-\phi(x)\bigr)=\phi^\star(g)-\phi^\star(x),
\qquad
\forall x,g\in \mathcal{X}.
\]
\end{definition}

\begin{definition}[Spectrum and spectral radius]
Let $T$ be a bounded linear operator on a Hilbert space $\mathcal H$. The
spectrum of $T$ is
\[
\sigma(T):=\{\lambda\in\mathbb C:T-\lambda I \text{ is not invertible}\}.
\]
The spectral radius of $T$ is $\rho(T):=\sup\{|\lambda|:\lambda\in\sigma(T)\}$.
\end{definition}
\begin{definition}[Orthogonal projection]
Let $\mathcal{H}$ be a Hilbert space and $\mathcal{U} \subseteq \mathcal{H}$ be a closed subspace. An {\bf orthogonal projection} onto $\mathcal{U}$ is a bounded linear operator $P: \mathcal{H} \to \mathcal{H}$ satisfying:
\begin{itemize}
    \item {\bf Idempotency (Projection):} $P^2 = P$.
    \item {\bf Range Constraint:} $\mathcal{R}(P) = \mathcal{U}$, where $\mathcal{R}(P) := \{ Ph : h \in \mathcal{H} \}$.
    \item {\bf Self-Adjointness (Orthogonality):} $\langle Ph, k \rangle = \langle h, Pk \rangle$ for all $h, k \in \mathcal{H}$.  
\end{itemize}
\end{definition}
\begin{definition}[Orthogonal complement]
For any subspace $\mathcal{U} \subseteq \mathcal{H}$, the orthogonal complement $\mathcal{U}^\perp$ is defined as the set of all vectors in $\mathcal{H}$ that are perpendicular to every vector in $\mathcal{U}$:
$$\mathcal{U}^\perp := \{ h \in \mathcal{H} : \langle h, u \rangle = 0 \quad \forall u \in \mathcal{U} \}$$
\end{definition}
\begin{definition}[Restriction]
Let $\mathcal{K}: \mathcal{H} \to \mathcal{H}$ be a linear operator. For any subspace $\mathcal{U} \subseteq \mathcal{H}$, the {\bf restriction of $\mathcal{K}$ to $\mathcal{U}$}, denoted $\mathcal{K} \big|_{\mathcal{U}}$, is a new mapping whose domain is limited to $\mathcal{U}$:
$\mathcal{K} \big|_{\mathcal{U}} : \mathcal{U} \to \mathcal{H}$
defined by the rule
$(\mathcal{K} \big|_{\mathcal{U}})(u) = \mathcal{K}u \quad \forall u \in \mathcal{U}.$
\end{definition}
\begin{definition}[Compression operator]
Let $\mathcal{H}$ be a Hilbert space and $P$ be an orthogonal projection onto a closed subspace $\mathcal{U} \subseteq \mathcal{H}$. For any bounded linear operator $\mathcal{K} \in \mathcal{L}(\mathcal{H})$, the {\bf compression} of $\mathcal{K}$ to $\mathcal{U}$ is the operator $\mathcal{K}_{\mathcal{U}}: \mathcal{U} \to \mathcal{U}$ defined by $\mathcal{K}_{\mathcal{U}} := P \mathcal{K} \big|_{\mathcal{U}}$.
Equivalently, viewing the operator on $\mathcal{H}$ with the property that it vanishes on $\mathcal{U}^\perp$, we write $\mathcal{K}_{\mathcal{U}} = P \mathcal{K} P.$
\end{definition}
%
$$\mathcal{K}_Q := P_Q \mathcal{K} \big|_{\mathcal{H}_Q}$$
%
\begin{definition}[Policy-induced scalar Markov operator]
Let $(\mathcal X,\mathcal A,\mathcal P)$ be a controlled Markov process and
let $\pi$ be a policy. The policy-induced transition kernel is
$\mathcal P^\pi(B\mid x):=\int_{\mathcal A}\mathcal P(B\mid x,a)\pi(da\mid x)$
for measurable $B\subseteq\mathcal X$. The associated scalar Markov operator
$P^\pi$ is defined by
\[
(P^\pi f)(x)
:=
\mathbb E_{x'\sim\mathcal P^\pi(\cdot\mid x)}[f(x')],
\]
for bounded measurable functions $f:\mathcal X\to\mathbb R$.
\end{definition}
\section{Auxiliary Lemmata}\label{sec:auxiliary_lemmata}
\begin{lemma}[Orthogonality of $P_Q$]
Let $\mathcal{H}_g = \text{span}(\phi(g))$ and $\mathcal{H}_Q = \mathcal{H}_g^\perp$. Let $P_Q: \mathcal{H} \to \mathcal{H}$ be the operator defined by the unique decomposition $h = h_Q + h_g$ where $h_Q \in \mathcal{H}_Q$ and $h_g \in \mathcal{H}_g$, such that $P_Q h = h_Q$. Then $P_Q$ is an {\bf orthogonal projection} onto $\mathcal{H}_Q$.   
\end{lemma}
\begin{proof}

To prove $P_Q$ is an orthogonal projection, we must verify the three conditions from the definition. For any $h \in \mathcal{H}$, let $P_Q h = h_Q$. By definition of the decomposition, $h_Q \in \mathcal{H}_Q$. Since $\mathcal{H}_Q$ is a subspace, the unique decomposition of $h_Q$ is $h_Q + 0$ (where $0 \in \mathcal{H}_g$). Thus:
$$P_Q(P_Q h) = P_Q(h_Q) = h_Q = P_Q h$$
This satisfies the projection property.

By definition, for any $h$, $P_Q h \in \mathcal{H}_Q$, so $\mathcal{R}(P_Q) \subseteq \mathcal{H}_Q$. Conversely, for any $v \in \mathcal{H}_Q$, its decomposition is $v + 0$, implying $P_Q v = v$. Thus $\mathcal{H}_Q \subseteq \mathcal{R}(P_Q)$. Therefore, $\mathcal{R}(P_Q) = \mathcal{H}_Q$.

We must show $\langle P_Q h, k \rangle = \langle h, P_Q k \rangle$ for arbitrary $h, k \in \mathcal{H}$. Let $h = h_Q + h_g$ and $k = k_Q + k_g$, with $h_Q, k_Q \in \mathcal{H}_Q$ and $h_g, k_g \in \mathcal{H}_g$. Recall that $\mathcal{H}_Q = \mathcal{H}_g^\perp$, meaning $\langle u, v \rangle = 0$ for any $u \in \mathcal{H}_Q, v \in \mathcal{H}_g$. We have for the left side$$\langle P_Q h, k \rangle = \langle h_Q, k_Q + k_g \rangle = \langle h_Q, k_Q \rangle + \langle h_Q, k_g \rangle$$
Since $h_Q \perp k_g$, this reduces to $\langle h_Q, k_Q \rangle$. We have for the right side
$$\langle h, P_Q k \rangle = \langle h_Q + h_g, k_Q \rangle = \langle h_Q, k_Q \rangle + \langle h_g, k_Q \rangle$$
Since $h_g \perp k_Q$, this reduces to $\langle h_Q, k_Q \rangle$. Since $\langle P_Q h, k \rangle = \langle h, P_Q k \rangle = \langle h_Q, k_Q \rangle$, the operator is self-adjoint. By the properties of projections in Hilbert spaces, a self-adjoint projection is an orthogonal projection.    
\end{proof}


\begin{lemma}[Poisson equation for the policy-induced chain]
Let $(\mathcal{X},\mathcal{A},\mathcal{P})$ be a Controlled Markov Process, fix a policy $\pi$, and let
$P^\pi$ denote the policy-induced transition operator on measurable functions:
$(P^\pi f)(x)
:=
\mathbb E_{x'\sim\mathcal P^\pi(\cdot\mid x)}[f(x')]$.
Fix a goal state $g\in\mathcal{X}$, and define the hitting time
$T_g^\pi:=\inf\{t\ge 0 : X_t=g\}$.
Assume $V^\pi(x,g):=\mathbb{E}[T_g^\pi\mid X_0=x]<\infty$
for all states $x$ under consideration.
Then $V^\pi(\cdot,g)$ satisfies the Poisson equation
$(I-P^\pi)V^\pi(\cdot,g)(x)=1$ for 
$x\in\mathcal{X}\setminus\{g\}$
with boundary condition $V^\pi(g,g)=0$.
\end{lemma}

\begin{proof}
Fix $x\neq g$. By conditioning on the first transition under $\pi$,
\[
V^\pi(x,g)
=
\mathbb{E}[T_g^\pi\mid X_0=x]
=
1+\mathbb{E}\!\left[V^\pi(X_1,g)\mid X_0=x\right].
\]
By definition of $P^\pi$, this is
\[
V^\pi(x,g)=1+(P^\pi V^\pi(\cdot,g))(x),
\]
hence
\[
(I-P^\pi)V^\pi(\cdot,g)(x)=1,
\qquad x\neq g.
\]
At the boundary, if $x=g$ then $T_g^\pi=0$ by definition, so $V^\pi(g,g)=0$.
\end{proof}

\begin{lemma}[Trajectory-level McDiarmid Concentration]
\label{lem:trajectory_level_mcdiarmid}
Let $\tau^1,\dots,\tau^N$ be independent trajectories, and let
$\mathcal G$ be a class of measurable real-valued functions on trajectories.
Assume that there exists $B<\infty$ such that
$0 \le g(\tau) \le B$
for all $g\in\mathcal G$ and all trajectories $\tau$.
Define
\[
Z(\tau^1,\dots,\tau^N)
:=
\sup_{g\in\mathcal G}
\left|
\mathbb E[g(\tau)]
-
\frac1N\sum_{j=1}^N g(\tau^j)
\right|.
\]
Then, for every $\delta\in(0,1)$, with probability at least $1-\delta$,
\[
Z(\tau^1,\dots,\tau^N)
\le
\mathbb E[Z(\tau^1,\dots,\tau^N)]
+
B\sqrt{\frac{\log(1/\delta)}{2N}}.
\]
\end{lemma}
\begin{proof}
We verify the bounded-differences condition at the level of whole
trajectories. Fix an index $j\in\{1,\dots,N\}$, and let
$\tau^{j\prime}$ be an alternative trajectory. Define
\[
\mathbf \tau
=
(\tau^1,\dots,\tau^j,\dots,\tau^N),
\qquad
\mathbf \tau'
=
(\tau^1,\dots,\tau^{j\prime},\dots,\tau^N),
\]
so that $\mathbf \tau$ and $\mathbf \tau'$ differ only in the $j$-th
trajectory. For any fixed $g\in\mathcal G$,
\[
\left|
\frac1N\sum_{i=1}^N g(\tau^i)
-
\frac1N
\left(
\sum_{i\ne j}g(\tau^i)+g(\tau^{j\prime})
\right)
\right|
=
\frac1N
\left|
g(\tau^j)-g(\tau^{j\prime})
\right|
\le
\frac{B}{N},
\]
because $0\le g(\tau)\le B$ for all trajectories $\tau$. Using the elementary inequality
\[
\left|\sup_{g\in\mathcal G} A_g-\sup_{g\in\mathcal G} B_g\right|
\le
\sup_{g\in\mathcal G}|A_g-B_g|,
\]
and applying the same bound inside the absolute value defining $Z$, we obtain $|Z(\mathbf \tau)-Z(\mathbf \tau')|
\le
\frac{B}{N}$.
Thus $Z$ satisfies the bounded-differences condition with constants $c_j=\frac{B}{N}$ for $j=1,\dots,N$. By McDiarmid's inequality, for every $t>0$,
\[
\mathbb P\left(
Z-\mathbb E[Z]\ge t
\right)
\le
\exp\left(
-\frac{2t^2}{\sum_{j=1}^N c_j^2}
\right).
\]
Since $\sum_{j=1}^N c_j^2
=
N\left(\frac{B}{N}\right)^2
=
\frac{B^2}{N}$,
we have $\mathbb P\left(
Z-\mathbb E[Z]\ge t
\right)
\le
\exp\left(
-\frac{2Nt^2}{B^2}
\right)$. Setting $t=
B\sqrt{\frac{\log(1/\delta)}{2N}}$
gives the desired result.
\end{proof}

\begin{lemma}[Talagrand Contraction for Bounded Squared Losses]
\label{lem:talagrand_squared_loss_contraction}
Let $\mathcal F$ be a class of real-valued functions on a domain $\mathcal Z$,
and fix a sample $(z_i,y_i)_{i=1}^n$. Assume that there exist constants
$B,H<\infty$ such that $|f(z_i)|\le B
\qquad\text{for all } f\in\mathcal F,\ i\in\{1,\dots,n\}$,
and $|y_i|\le H$ for all $i\in\{1,\dots,n\}$.
Then
\[
\mathbb E_\sigma
\left[
\sup_{f\in\mathcal F}
\frac1n\sum_{i=1}^n
\sigma_i (f(z_i)-y_i)^2
\right]
\le
2(B+H)
\,
\mathbb E_\sigma
\left[
\sup_{f\in\mathcal F}
\frac1n\sum_{i=1}^n
\sigma_i f(z_i)
\right],
\]
where $\sigma_1,\dots,\sigma_n$ are independent Rademacher random variables.
\end{lemma}
\begin{proof}
For each $i\in\{1,\dots,n\}$, define $\psi_i(u):=(u-y_i)^2-y_i^2$. Then $\psi_i(0)=0$. Moreover, on the interval $[-B,B]$, it holds that $|\psi_i'(u)|
=
2|u-y_i|
\le
2(B+H)$, because $|u|\le B$ and $|y_i|\le H$. Hence each $\psi_i$ is
$2(B+H)$-Lipschitz on the range of the functions in $\mathcal F$. Now observe that $(f(z_i)-y_i)^2 = \psi_i(f(z_i))+y_i^2$.
Therefore,
\[
\frac1n\sum_{i=1}^n
\sigma_i (f(z_i)-y_i)^2
=
\frac1n\sum_{i=1}^n
\sigma_i \psi_i(f(z_i))
+
\frac1n\sum_{i=1}^n
\sigma_i y_i^2 .
\]
The second term is independent of $f$, so
\[
\sup_{f\in\mathcal F}
\frac1n\sum_{i=1}^n
\sigma_i (f(z_i)-y_i)^2
=
\sup_{f\in\mathcal F}
\frac1n\sum_{i=1}^n
\sigma_i \psi_i(f(z_i))
+
\frac1n\sum_{i=1}^n
\sigma_i y_i^2 .
\]
Taking expectation over the Rademacher signs, the last term vanishes:
\[
\mathbb E_\sigma
\left[
\frac1n\sum_{i=1}^n
\sigma_i y_i^2
\right]
=0.
\]
Hence
\[
\mathbb E_\sigma
\left[
\sup_{f\in\mathcal F}
\frac1n\sum_{i=1}^n
\sigma_i (f(z_i)-y_i)^2
\right]
=
\mathbb E_\sigma
\left[
\sup_{f\in\mathcal F}
\frac1n\sum_{i=1}^n
\sigma_i \psi_i(f(z_i))
\right].
\]

By the Talagrand contraction principle for Rademacher averages, since each
$\psi_i$ satisfies $\psi_i(0)=0$ and is $2(B+H)$-Lipschitz on the relevant
range,
\[
\mathbb E_\sigma
\left[
\sup_{f\in\mathcal F}
\frac1n\sum_{i=1}^n
\sigma_i \psi_i(f(z_i))
\right]
\le
2(B+H)
\,
\mathbb E_\sigma
\left[
\sup_{f\in\mathcal F}
\frac1n\sum_{i=1}^n
\sigma_i f(z_i)
\right].
\]
Combining the previous displays proves the claim.
\end{proof}
\input{proofs}


\newpage
\input{algorithms/IEL}

\input{algorithms/InfoNCE}
\input{algorithms/three_phase_training}

\input{algorithms/asymmetric_graph_planning}

\newpage
\input{tables/main_configuration}
\input{tables/many_configurations_results}






\end{document}

%% file: algorithm_section.tex
\subsection{The Algorithm: Isomorphic Embedding Learning (IEL)}

We next introduce a new foundation policy learning algorithm building on the theoretical framework developed earlier. Our algorithm, named as \emph{Isomorphic Embedding Learning (IEL)}, comprises the four components explained below. The related pseudocode can be found in Appendix \ref{sec:algorithm_pseudocode}.

(i) {\bf Task identifier learning:} We learn a mapping $\omega_\psi(\cdot)$ that transforms a target state $g$ into a task identifier $\omega_g^\pi$, which defines the directional axis of the latent space for that specific navigation objective. This is essential for implementing the directional asymmetry in the hitting-time calculation $V^\pi(x, g) = \langle \phi(g) - \phi(x), \omega_g^\pi \rangle_{\mathcal{H}}$ established in Theorem \ref{thm:existence_of_linear_hitting_time_representation}. We adopt a contrastive learning approach to implement this step by computing an InfoNCE loss that contrasts the task identifier $\omega_\psi(g_i)$ against a corresponding augmented positive view $\omega_\psi(\tilde g_i)$ and a set of negative candidates within a batch. Specifically, for a batch of goals $\{g_i\}_{i=1}^N$ sampled from the replay buffer, where $\tilde g_i$ denotes the augmented view paired with $g_i$ as in Algorithm~\ref{alg:InfoNCE}, we minimize:
\begin{align*}
\mathcal{L}_{\text{NCE}}(\psi)
=
-\frac{1}{N}\sum_{i=1}^N
\left[
\frac{\omega_\psi(g_i)^\top \omega_\psi(\tilde g_i)}{\tau}
-
\log \sum_{j=1}^N
\exp\left(
\frac{\omega_\psi(g_i)^\top \omega_\psi(\tilde g_j)}{\tau}
\right)
\right].
\end{align*}
The denominator contrasts the paired positive view against other batch
candidates, encouraging augmented views of the same state to have similar identifiers while separating distinct states. This stage should be interpreted as learning an injective task parameterization, not as directly solving the policy-specific Poisson
equation for the exact representers $\omega_g^\pi$ in Theorem~2.2. In the finite-sample analysis, the learned task encoder is treated as fixed and denoted by $\bar \omega_\psi$. Any mismatch between this learned identifier and an ideal hitting-time readout is absorbed into the approximation/oracle-competitiveness terms of the regression analysis.

(ii) {\bf Hitting time regression:} We sample an intermediate point from a trajectory. We then learn to predict the hitting-time progress supported by the behavior data. Following the lines of IQL \citep{kostrikov2022iql}, we perform expectile regression to
bias the learned geometry toward shorter in-distribution distances from
observations. To keep the empirical hitting-time target bounded and stable, we implement discounted hitting times in line with common practice in sufficient-MDP-based approaches.

(iii) {\bf Temporal difference learning.} The Monte Carlo approach used in the previous component is effective for learning to predict hitting times in an unbiased manner. However, collecting long-horizon hitting-time observations is difficult: long-range data is rare, often following a power-law tail, and such targets become increasingly hard to predict accurately. Instead, we restrict the interim goal horizon and support the stitching capabilities of our approach via TD learning. The detrimental effect of the limited horizon of trajectories on stitching capabilities has been observed earlier \citep{myers2025tmd}. A direct Bellman backup implementation of the hitting time machinery is susceptible to numerical instabilities. To account for them, we first establish an analytical link between our hitting time formula and HILP, noting that HILP's $L_2$ norm formulation appears as an upper bound to ours. We bridge these two models by grounding the latent displacement in the projection of the embedding difference onto the canonical unit vector. This projection represents the directed progress toward the goal state. By normalizing the task identifier we get
\begin{align*}
    \langle \phi(g) - \phi(x), \omega_\psi(g) / \|\omega_\psi(g)\| \rangle = \|\phi(g) - \phi(x)\| \underbrace{ \frac{\langle \phi(g) - \phi(x), \omega_\psi(g) \rangle}{\|\phi(g) - \phi(x)\| \cdot \|\omega_\psi(g)\|} }_{=\cos(\xi_{x\rightarrow g, g})}
\end{align*}
where $\xi_{x\rightarrow g, g}$ is the angle between the embedding vectors of $s$ and $g$. We further employ a  transformation  $V^\pi(x,g) \approx \|\phi(g) - \phi(x)\|  \exp \left( \beta (1 - \cos(\xi_{x\rightarrow g, g})) \right)$ which penalizes both the Euclidean distance and the angular misalignment relative to the goal, introducing a directional bias that symmetric HILP objective lacks. The exponential factor is monotone in $\cos(\xi_{x\rightarrow g, g})$: it equals one when $\phi(g) - \phi(s)$ is aligned with $\omega_\psi(g)$ and grows smoothly as alignment degrades, so it acts as a soft directional constraint that augments the underlying Euclidean term. We adopt this transformation as a smoother proxy for the directed hitting-time score. Notably, $\beta=0$ recovers the HILP objective.

(iv) {\bf Graph search for multi-stage planning.} We exploit the learned asymmetric geometry by constructing a directed graph $G=(V,E)$. The vertex set $V\subset \{\phi(s) | s \in \mathcal{B} \}$ is extracted from the dataset using a greedy approximation for Determinantal Point Process inference \cite{kulesza2012determinantal}, ensuring diverse coverage of the state manifold. The edge set $E$ is initialized via a symmetrized Minimum Spanning Tree (MST) based on $d_{sym}=d(s,t,g)+d(t,s,g)$, ensuring strong connectivity. 
This backbone is augmented with $k$-Nearest-Neighbors ($k$-NN) for each vertex. While the MST backbone is determined by $d_{sym}$, all edge costs retain their asymmetric costs defined in \ref{eq:edge_cost}, leveraging the learned quasimetric isomorphism of the state space; induced graph distances respect the triangle inequality:
\begin{align}
d(s,t, g) :=\| \phi(t) - \phi(s)\|_2 \cdot \exp ( \beta ( 1 - \cos(\xi_{s\rightarrow t, g}))).
\label{eq:edge_cost}
\end{align}
Navigation is performed by computing the Single Source Shortest Path (SSSP) on the reversed graph $\bar{G}$ from the goal vertex $\phi(g)$. This yields a global navigation function. At evaluation time, the agent identifies the nearest vertex $s\in V$ and pursues the subsequent node on the optimal path toward $g$. See Algorithm \ref{alg:asymmetric_graph_planning} for the full pseudocode. Note that a symmetric variant of the asymmetric graph planning algorithm can be derived by using a symmetric distance function (e.g., $d(s,t)=\|\phi(t)-\phi(s)\|_2$) and undirected edges for the graph. This reduces the computational overhead significantly, as the graph and shortest path can be precomputed before evaluation, as the distance function is goal-agnostic.

%% file: experiments.tex
\section{Experiments}

We evaluate IEL on six offline goal-conditioned RL benchmark environments from two different suites taken from the D4RL data set \citep{fu2020d4rl}. HILP serves as our primary baseline, being the only prior art designed for foundation policy training and multi-stage planning. We strictly adhere to the planning experimental setup and hyperparameters reported in HILP (Table \ref{table:main_configuration}) to ensure a direct comparison. As practiced in HILP experiments, we do not consider approaches such as CRL, QRL, TMD, and GAS as fair baselines since their pipelines are tuned for goal-supervised training rather than foundation policy learning. We compare against GC-IQL and GC-CQL as the only representatives of zero-shot RL for which HILP reported complete results.  To isolate the effect of our isomorphic embeddings, we evaluate both HILP and IEL using our graph-based planning approach. Specifically, we compare IEL under \emph{symmetric} ($\beta=0$) and \emph{asymmetric} ($\beta=0.1$) distance scores. HILP+Sym is the closest GAS-equivalent baseline for our problem setting. Like GAS, it builds on HILP-style representations, but omits the hierarchical supervised subgoal-reaching objective (see Eqs.~13--14 of \cite{baek2025gas}). Instead, HILP+Sym isolates the graph-planning ingredient relevant to foundation policy training by applying symmetrized graph construction and shortest-path subgoal selection on top of HILP. Each algorithm navigates the shortest path extracted from the planning graph using its respective distance metric. Thus, the symmetric variant of IEL still benefits from directional identifiers during the planning process.  We report our main empirical results in Table \ref{table:main_results}, which reveals two critical outcomes: (i) IEL leverages graph-based planning more effectively than HILP, with its symmetric graph performance surpassing recursive midpoint planning by a significant margin in five environments. (ii) asymmetric planning yields further significant gains for IEL in four environments and causes minor degradation in the remaining two environments. These results validate our core hypothesis: \emph{directional embedding learning is essential for effective multi-stage planning in foundation policies.}  IEL's performance is mildly sensitive to its hyperparameters $H_{\max}$, $\beta$, and $\sigma$. See Appendix \ref{sec:experiment_details} for a discussion and further experiment details. IEL's offline training-time computational cost is ignorably different from HILP. We set the coreset size to be $8192$, which is the largest size that maintains a wall-clock time proportional to the \emph{Rec-Mid} evaluation period. In this setting, \emph{Asym-Graph} runs at a comparable speed to \emph{Rec-Mid}, while \emph{Sym-Graph} is approximately three times faster due to graph reusability.

\input{tables/main_results}

%% file: tables/main_results.tex
\begin{table}[t!]
\centering
\small
\setlength{\tabcolsep}{3pt}
\caption{
Mean normalized return $\pm$ standard deviation for offline zero-shot goal-conditioned RL. All experiments are based on eight seeds trained for a total of one million gradient steps. GC-IQL and GC-CQL results are from \cite{park2024hilp}. 
\emph{Rec-Mid} is HILP's planning method that recursively selects midpoint between current state and goal for three recursions. 
\emph{Sym-Graph} stands for graph planning using symmetric Euclidean distances and \emph{Asym-Graph} for graph planning using asymmetric planning.}
\label{table:main_results}
\resizebox{\textwidth}{!}{
\begin{tabular}{lccccccc}
\toprule
& \multicolumn{2}{c}{ZSRL Algorithms} & \multicolumn{2}{c}{HILP \cite{park2024hilp}} & \multicolumn{3}{c}{IEL (Ours)} \\
\cmidrule(lr){2-3} \cmidrule(lr){4-5} \cmidrule(lr){6-8}
Dataset & GC-IQL & GC-CQL & Rec-Mid & Sym-Graph & Rec-Mid & Sym-Graph & Asym-Graph \\
\midrule
antmaze-large-diverse & 56.0 \std{6.0}  & 36.2 \std{19.0} & 68.0 \std{6.3} & 60.5 \std{18.8} & 63.8 \std{10.8} & 55.0 \std{20.6} & \textbf{71.8 \std{6.5}} \\
antmaze-large-play     & 56.0 \std{25.7} & 32.0 \std{25.8} & 54.5 \std{11.9} & 45.5 \std{8.3} & 59.2 \std{7.2} & \textbf{63.0 \std{14.9}} & 62.8 \std{5.9} \\
antmaze-ultra-diverse & 40.8 \std{11.1} & 14.2 \std{13.5} & 39.8 \std{22.1} & 68.8 \std{23.0} & 52.5 \std{17.3} & 62.5 \std{18.7} & \textbf{79.0 \std{4.5}} \\
antmaze-ultra-play    & 41.8 \std{9.0}  & 16.5 \std{14.3} & 56.0 \std{18.9} & 60.2 \std{10.4} & 55.8 \std{7.1} & 65.2 \std{13.5} & \textbf{73.8 \std{6.2}} \\
kitchen-partial       & 56.4 \std{8.4}  & 31.2 \std{16.6} & 55.4 \std{9.2} & 53.8 \std{8.1} & 48.6 \std{12.6} & \textbf{60.4 \std{5.5}} & \textbf{60.4 \std{3.5}} \\
kitchen-mixed         & 59.5 \std{3.8}  & 15.7 \std{17.6} & 53.7 \std{3.5} & 49.0 \std{4.6} & 52.3 \std{7.8} & \textbf{57.8 \std{5.1}} & 57.2 \std{3.8} \\
\bottomrule
\end{tabular}
}
\end{table}

%% file: conclusion.tex
\section{Open Questions and Scientific Impact}\label{sec:open_questions}

The theoretical framework we built to set a foundation for the IEL algorithm unfolds new open questions. The range condition defined in Theorems \ref{thm:existence_of_linear_hitting_time_representation} and \ref{thm:identifiability_of_linear_hitting_time_representation} requires rigorous verification across non-linear environments to ensure universal representability. While identifiability is established up to a linear isomorphism, the impact of state-space discretization on directed geometry in high-dimensional settings remains unquantified. Consistency analysis reveals that global hitting-time error scales by the factor $1/(1-\rho(\mathcal{T}_Q^\pi))$, creating high sensitivity to the transient spectral radius. The $\epsilon$-Sufficient CMP assumption (Definition \ref{def:cmp}) restricts the framework to environments where expected feature transitions are sufficiently linearizable. This formulation fails to cover dynamics with high-order stochasticity or non-Markovian dependencies that exceed the latent Hilbert space span. On the practical side, the upper bound used in the TD learning part can be tightened so that the directionality of the learned embeddings can be strengthened. This work provides the mathematical foundation to transition VLA research from its current reliance on imitation learning toward reward-based policy learning. By replacing symmetric distance heuristics with a directed temporal geometry, our framework enables agents to navigate complex, irreversible environments with provable efficiency. The outcome could be versatile general-purpose robots with enhanced autonomy and maintained trustworthiness, supported by the rigorous operator-theoretic guarantees of our framework.

%% file: proofs.tex
\input{unsupervised-policy-training}

\section{Problem span}
\label{app:problem_span}

We next clarify when the learned representation can be reused beyond the specific hitting-time regression problem used to train it. A suitable Hilbert representation supports goal-conditioned hitting-time geometry through displacement readouts of the form $V^\pi(x,g) = \langle \phi(g)-\phi(x),\omega_g^\pi\rangle_{\mathcal H}$. A natural question is therefore: \emph{which other control tasks can be solved using the same representation, rather than relearning the geometry?} We refer to this family of compatible tasks as the \emph{problem span} of the representation. We aim to identify sufficient conditions under which the same latent geometry can support additional rewards and value functions. We distinguish three notions: (i) \emph{reward representability}: whether the immediate reward can be read out from the learned latent geometry; (ii) \emph{policy-evaluation closure}: whether the value function of a fixed policy remains in the latent linear span; (iii) \emph{optimal-control closure}: whether the Bellman optimality operator preserves the latent value class. The first two follow from natural linear-operator assumptions. The third depends on whether  maximization over actions preserves linearity.

\subsection{Latent reward span}

We first define the class of immediate rewards that are directly represented by the learned geometry.

\begin{definition}[Latent bilinear reward span]
\label{def:latent_bilinear_reward_span} Given a representation $(\mathcal H,\phi,\{\mathcal K_a\}_{a\in\mathcal A})$, a goal-conditioned reward $r:\mathcal X\times\mathcal A\times\mathcal X\to\mathbb R$ belongs to the latent bilinear reward span if, for every action $a\in\mathcal A$, there exists a bounded linear operator
$\mathcal R_a\in\mathcal L(\mathcal H)$ such that $
r(x,a,g) = \langle \phi(x),\mathcal R_a\phi(g)\rangle_{\mathcal H}$.
\end{definition}

This definition says that the reward can be computed from an interaction between the state embedding $\phi(x)$ and the goal embedding $\phi(g)$, possibly transformed by an action-dependent task operator $\mathcal R_a$. It is a representability condition on the one-step reward only. By itself it does not yet imply that the corresponding optimal value function is bilinear. Several common goal-conditioned reward families fit this template.

{\bf Kernel goal similarity.} If $\mathcal R_a=\alpha_a I$ for scalars $\alpha_a$, then $r(x,a,g)=\alpha_a\langle \phi(x),\phi(g)\rangle_{\mathcal H}$. This is a smooth goal-similarity reward. It should be interpreted as a soft similarity objective, not as an exact hitting-time objective. Exact hitting-time objectives are governed by the Poisson/range-condition results in the main paper.

{\bf Weighted or anisotropic goal similarity.} If $\mathcal R_a$ is positive semidefinite or low rank, then the reward compares $x$ and $g$ only along selected latent directions: $r(x,a,g)=\langle \phi(x),\mathcal R_a\phi(g)\rangle_{\mathcal H}$.
This covers tasks where different actions or task modes emphasize different aspects of the goal, such as position, object identity, safety-relevant features, or other latent factors.

{\bf Region and fuzzy goals.} A goal need not be a single state. Let $\rho_g$ be a distribution over acceptable goal states and define its latent mean embedding
$\mu_g:=\mathbb E_{Y\sim\rho_g}[\phi(Y)]$. Then rewards of the form $r(x,a,g) = \langle \phi(x),\mathcal R_a\mu_g\rangle_{\mathcal H}$ are also in the problem span. This covers fuzzy goals, goal regions, and safety corridors, provided that the relevant region can be summarized by its latent mean embedding.

{\bf Compositional goals.} If a downstream task is described by a finite combination of goals $g_1,\ldots,g_m$, one can define $\mu:=\sum_{j=1}^m \alpha_j\phi(g_j)$ and use $r(x,a,\mu)=\langle \phi(x),\mathcal R_a\mu\rangle_{\mathcal H}$. Thus the representation supports linear composition of goal specifications in
latent space.

\subsection{Fixed-policy evaluation span}

We next show that latent reward representability becomes a value-function representation result for fixed-policy evaluation, provided that the transition dynamics close linearly in the same representation.

\begin{proposition}[Fixed-policy evaluation span]
\label{prop:fixed_policy_problem_span}
Let $(\mathcal X,\mathcal A,\mathcal P)$ be a sufficient CMP with representation
$(\mathcal H,\phi,\{\mathcal K_a\}_{a\in\mathcal A})$, and assume
$\sup_{x\in\mathcal X}\|\phi(x)\|_{\mathcal H}<\infty$. Fix a policy $\pi$ and a
discount factor $\gamma\in(0,1)$. Assume that the policy-induced latent
transition operator $\mathcal K_\pi$ is bounded and satisfies
$\mathbb E_{x'\sim \mathcal P^\pi(\cdot\mid x)}[\phi(x')]
=
\mathcal K_\pi\phi(x)$.
Assume also that the policy-induced reward is latent-linear: there exists a
bounded linear operator $\mathcal R_\pi\in\mathcal L(\mathcal H)$ such that $r_\pi(x,g)
:=
\mathbb E_{a\sim\pi(\cdot\mid x)}[r(x,a,g)]
=
\langle \phi(x),\mathcal R_\pi\phi(g)\rangle_{\mathcal H}$.
If $I-\gamma\mathcal K_\pi^*$ is invertible with bounded inverse, then the
discounted value function of $\pi$ admits the bilinear representation $V^\pi(x,g)
=
\langle \phi(x),\mathcal S_\pi\phi(g)\rangle_{\mathcal H}$,
where $\mathcal S_\pi
:=
(I-\gamma\mathcal K_\pi^*)^{-1}\mathcal R_\pi$.
\end{proposition}

\begin{proof}
Fix a goal $g$ and define $\omega_g^\pi
:=
(I-\gamma\mathcal K_\pi^*)^{-1}\mathcal R_\pi\phi(g)$.
Let $\widetilde V^\pi(x,g)
:=
\langle \phi(x),\omega_g^\pi\rangle_{\mathcal H}$.
By definition we get $\omega_g^\pi
=
\mathcal R_\pi\phi(g)+\gamma\mathcal K_\pi^*\omega_g^\pi$. Therefore, for every $x\in\mathcal X$,
\[
\begin{aligned}
\widetilde V^\pi(x,g)
&=
\langle \phi(x),\mathcal R_\pi\phi(g)\rangle_{\mathcal H}
+
\gamma
\langle \phi(x),\mathcal K_\pi^*\omega_g^\pi\rangle_{\mathcal H} \\
&=
r_\pi(x,g)
+
\gamma
\langle \mathcal K_\pi\phi(x),\omega_g^\pi\rangle_{\mathcal H} \\
&=
r_\pi(x,g)
+
\gamma
\mathbb E_{x'\sim \mathcal P^\pi(\cdot\mid x)}
[
\langle \phi(x'),\omega_g^\pi\rangle_{\mathcal H}
] \\
&=
r_\pi(x,g)
+
\gamma
\mathbb E_{x'\sim \mathcal P^\pi(\cdot\mid x)}
[
\widetilde V^\pi(x',g)
].
\end{aligned}
\]
Thus $\widetilde V^\pi(\cdot,g)$ satisfies the discounted Bellman equation for policy evaluation. Since $\phi$ is bounded and $\mathcal R_\pi$ and $(I-\gamma\mathcal K_\pi^*)^{-1}$ are bounded, $\widetilde V^\pi(\cdot,g)$ is bounded. The discounted Bellman operator is a $\gamma$-contraction on bounded measurable functions, so its bounded fixed point is unique. Hence $\widetilde V^\pi(x,g)=V^\pi(x,g)$. Defining $\mathcal S_\pi:=(I-\gamma\mathcal K_\pi^*)^{-1}\mathcal R_\pi$
gives $V^\pi(x,g)
=
\langle \phi(x),\mathcal S_\pi\phi(g)\rangle_{\mathcal H}$.
\end{proof}

The assumption that $r_\pi$ is represented by a single operator
$\mathcal R_\pi$ is structurally guaranteed in some simple cases, for example when the policy uses a state-independent mixture over action-specific reward operators. In general, however, action probabilities may depend on $x$, so this assumption should be understood as a policy-level reward closure condition.

\begin{proposition}[Action-value span]
\label{prop:action_value_problem_span}
Under the assumptions of Proposition~\ref{prop:fixed_policy_problem_span},
assume additionally that the action-specific rewards and transitions satisfy $r(x,a,g)
=
\langle \phi(x),\mathcal R_a\phi(g)\rangle_{\mathcal H}$ and $\mathbb E_{x'\sim \mathcal P(\cdot\mid x,a)}[\phi(x')]
=
\mathcal K_a\phi(x)$.
Then the action-value function of $\pi$ admits the representation
\[
Q^\pi(x,a,g)
=
\left\langle
\phi(x),
\left(\mathcal R_a+\gamma\mathcal K_a^*\mathcal S_\pi\right)\phi(g)
\right\rangle_{\mathcal H}.
\]
\end{proposition}

\begin{proof}
By definition, $Q^\pi(x,a,g)
=
r(x,a,g)
+
\gamma
\mathbb E_{x'\sim \mathcal P(\cdot\mid x,a)}[V^\pi(x',g)]$.
Using Proposition~\ref{prop:fixed_policy_problem_span}, we get $V^\pi(x',g)
=
\langle \phi(x'),\mathcal S_\pi\phi(g)\rangle_{\mathcal H}$.
Therefore
\[
\begin{aligned}
Q^\pi(x,a,g)
&=
\langle \phi(x),\mathcal R_a\phi(g)\rangle_{\mathcal H}
+
\gamma
\mathbb E_{x'\sim \mathcal P(\cdot\mid x,a)}
[
\langle \phi(x'),\mathcal S_\pi\phi(g)\rangle_{\mathcal H}
] \\
&=
\langle \phi(x),\mathcal R_a\phi(g)\rangle_{\mathcal H}
+
\gamma
\langle \mathcal K_a\phi(x),\mathcal S_\pi\phi(g)\rangle_{\mathcal H} \\
&=
\left\langle
\phi(x),
\left(\mathcal R_a+\gamma\mathcal K_a^*\mathcal S_\pi\right)\phi(g)
\right\rangle_{\mathcal H}.
\end{aligned}
\]
\end{proof}

\subsection{Optimal-control span}

The previous propositions concern policy evaluation. Optimal control requires
additional care. Even when each action-value candidate is bilinear, the optimal
value is obtained by a pointwise maximization: $V^*(x,g)=\max_{a\in\mathcal A}Q^*(x,a,g)$.
A maximum of linear or bilinear functions is not generally linear. Therefore, latent bilinear rewards and latent transition closure do not by themselves imply that $V^*$ is bilinear. We can nevertheless define the additional closure condition under which optimal control remains inside the problem span.

\begin{definition}[Latent optimal Bellman closure]
\label{def:latent_optimal_bellman_closure}
Let
\[
\mathcal V_\phi
:=
\left\{
V_S:\mathcal X\times\mathcal X\to\mathbb R
\;:\;
V_S(x,g)=\langle \phi(x),S\phi(g)\rangle_{\mathcal H},
\quad S\in\mathcal L(\mathcal H)
\right\}.
\]
A discounted goal-conditioned control problem is said to be optimally closed in
the representation $\phi$ if its optimal Bellman operator $\mathcal B^*$ maps
$\mathcal V_\phi$ into itself: $\mathcal B^*\mathcal V_\phi\subseteq\mathcal V_\phi$.
\end{definition}

\begin{proposition}[Optimal-control span under Bellman closure]
\label{prop:optimal_control_problem_span}
Assume that the discounted goal-conditioned control problem is optimally closed
in $\mathcal V_\phi$, i.e., $\mathcal B^*\mathcal V_\phi\subseteq \mathcal V_\phi$. Assume further that $\mathcal B^*$ is a $\gamma$-contraction on bounded measurable functions, that $\mathcal V_\phi$ is closed in the sup norm, and that $0\in\mathcal V_\phi$. Then the optimal value function belongs to $\mathcal V_\phi$. That is, there
exists $S^*\in\mathcal L(\mathcal H)$ such that $V^*(x,g)
=
\langle \phi(x),S^*\phi(g)\rangle_{\mathcal H}$.
If, in addition, action-specific rewards and transitions are represented by operators $\mathcal R_a$ and $\mathcal K_a$, then
$Q^*(x,a,g)
=
\left\langle
\phi(x),
\left(\mathcal R_a+\gamma\mathcal K_a^*S^*\right)\phi(g)
\right\rangle_{\mathcal H}$.
\end{proposition}

\begin{proof}
Let $V_0=0$. Since $0\in\mathcal V_\phi$ and
$\mathcal B^*\mathcal V_\phi\subseteq\mathcal V_\phi$, the iterates
$V_{k+1}:=\mathcal B^*V_k$ satisfy $V_k\in\mathcal V_\phi$ for all $k\ge 0$. Since $\mathcal B^*$ is a $\gamma$-contraction on bounded measurable functions, $V_k$ converges in supremum norm to the unique optimal value function $V^*$. Because $\mathcal V_\phi$ is closed in the sup norm, $V^*\in\mathcal V_\phi$. Hence there exists $S^*\in\mathcal L(\mathcal H)$ such that $V^*(x,g)=\langle \phi(x),S^*\phi(g)\rangle_{\mathcal H}$. The expression for $Q^*$ follows from the one-step optimal action-value identity:
\[
Q^*(x,a,g)
=
r(x,a,g)
+
\gamma
\mathbb E_{x'\sim \mathcal P(\cdot\mid x,a)}[V^*(x',g)].
\]
Substituting the representations of $r$, $\mathcal P$, and $V^*$ gives
\[
Q^*(x,a,g)
=
\left\langle
\phi(x),
\left(\mathcal R_a+\gamma\mathcal K_a^*S^*\right)\phi(g)
\right\rangle_{\mathcal H}.
\]
\end{proof}

We adopt optimal closure as an assumption, because the condition $\mathcal B^*\mathcal V_\phi\subseteq\mathcal V_\phi$ is not guaranteed by construction. It rules out cases where the maximization over actions creates a piecewise-linear or otherwise nonlinear envelope outside the latent bilinear class. Thus the representation supports optimal control for tasks whose Bellman updates are closed in the learned span, but latent bilinear rewards alone do not guarantee this closure.

\subsection{Relation to hitting-time geometry}

The hitting-time objectives studied in the main paper are not merely discounted latent-bilinear reward problems. They are undiscounted stopping-time objectives with boundary condition $V^\pi(g,g)=0$ and unit cost before hitting the goal. Their linear representation is governed by a Poisson equation and the range condition $\mathbf w_{\mathbf 1} \in \operatorname{range}((\mathcal K_\pi-I)^*)$, as established in Theorem~\ref{thm:existence_of_linear_hitting_time_representation}.
The approximation theory in Theorem~\ref{thm:approximation_error} then shows that compressed latent transition consistency yields controlled error in the corresponding hitting-time predictor. Thus the paper contains two complementary notions of problem span:
(i) a \emph{hitting-time span}, controlled by the Poisson/range-condition theory and used directly by IEL; and (ii) a \emph{discounted reward span}, where additional latent-bilinear rewards can be evaluated or optimized when the relevant Bellman closure conditions hold.

\subsection{When should the representation work?}

IEL-style representations are useful under the following conditions:  \emph{(i) Predictive latent closure.} The representation should make one-step latent prediction approximately linear: $\mathbb E[\phi(X')\mid x,a]\approx \mathcal K_a\phi(x)$. If important transition-relevant information is absent from $\phi(x)$, then no linear readout can recover it downstream.
 \emph{(ii) Goal information is represented in the latent space.}
The downstream task should depend on goals through features present in $\phi(g)$, or through simple combinations such as mean embeddings or finite linear mixtures of goal embeddings. If the task depends on goal attributes that are not encoded by $\phi$, it lies outside the problem span.  \emph{(iii) Rewards or costs admit latent readouts.}
For discounted reward tasks, the immediate reward should be expressible, at least approximately, as $r(x,a,g)\approx \langle \phi(x),\mathcal R_a\phi(g)\rangle_{\mathcal H}$.
For hitting-time tasks, the unit-cost representer and Poisson/range condition play the analogous role.  \emph{(iv) The effective horizon is controlled.} Local representation errors accumulate through the relevant resolvent. Long horizons, nearly noncontractive transient dynamics, or poor goal reachability can amplify small one-step errors into large value errors. \emph{(v) Optimal control may require additional closure.} For policy evaluation, latent transition and reward closure are enough. For optimal control, the maximization over actions must also preserve the latent value class, or one should interpret the representation as supporting approximate rather than exact optimal planning.  All in all, the learned IEL geometry provides a shared displacement and transition structure rather than a universal reward solver. It enables reuse across goal-conditioned tasks where rewards, boundary conditions, and Bellman updates remain compatible with this established geometry.

\section{Main Results}\label{sec:main_results}

\subsection{Proof of Theorem \ref{thm:existence_of_linear_hitting_time_representation}}
\begin{proof}
Fix a goal $g$ and a policy $\pi$. Let $P^\pi$ denote the scalar Markov
operator $(P^\pi f)(x)
:=
\mathbb E_{x'\sim \mathcal P^\pi(\cdot\mid x)}[f(x')]$.
We prove both directions.

\paragraph{Necessity.}
Assume that there exists $\omega_g^\pi\in\mathcal H$ such that
$V^\pi(x,g)
=
\langle \phi(g)-\phi(x),\omega_g^\pi\rangle_{\mathcal H}$.
On the transient set $\mathcal X\setminus\{g\}$, the hitting-time function
satisfies the Poisson equation
$(I-P^\pi)V^\pi(\cdot,g)(x)=1$.
Substituting the linear representation gives, for $x\neq g$,
\[
\begin{aligned}
1
=
(I-P^\pi)
\bigl(
\langle \phi(g),\omega_g^\pi\rangle_{\mathcal H}
-
\langle \phi(\cdot),\omega_g^\pi\rangle_{\mathcal H}
\bigr)(x) 
=
-\langle \phi(x),\omega_g^\pi\rangle_{\mathcal H}
+
\mathbb E_{x'\sim \mathcal P^\pi(\cdot\mid x)}
\bigl[
\langle \phi(x'),\omega_g^\pi\rangle_{\mathcal H}
\bigr].
\end{aligned}
\]
By sufficiency of the CMP, $\mathbb E_{x'\sim \mathcal P^\pi(\cdot\mid x)}[\phi(x')]
=
\mathcal K_\pi\phi(x)$.
Therefore $1
=
\langle (\mathcal K_\pi-I)\phi(x),\omega_g^\pi\rangle_{\mathcal H}
=
\langle \phi(x),(\mathcal K_\pi-I)^*\omega_g^\pi\rangle_{\mathcal H}$.
Since $\mathbf w_{\mathbf 1}$ is the unit-cost representer, $\langle \phi(x),\mathbf w_{\mathbf 1}\rangle_{\mathcal H}=1$ for $x\neq g$.
Hence $\langle \phi(x),(\mathcal K_\pi-I)^*\omega_g^\pi\rangle_{\mathcal H}
=
\langle \phi(x),\mathbf w_{\mathbf 1}\rangle_{\mathcal H}$ for $x\neq g$. By the closed-feature-span convention for hitting-time identities, equality
against all $\phi(x)$ with $x\neq g$ identifies elements on the relevant
transient feature span. Thus $(\mathcal K_\pi-I)^*\omega_g^\pi
=
\mathbf w_{\mathbf 1}$
on the relevant transient feature span. Consequently, $\mathbf w_{\mathbf 1}
\in
\operatorname{range}((\mathcal K_\pi-I)^*)$
in the sense of this transient-span convention.

\paragraph{Sufficiency.}
Assume now that $\mathbf w_{\mathbf 1}
\in
\operatorname{range}((\mathcal K_\pi-I)^*)$
where the range condition is understood on the relevant transient feature
span for goal $g$. Then there exists $\omega_g^\pi\in\mathcal H$ such that $(\mathcal K_\pi-I)^*\omega_g^\pi
=
\mathbf w_{\mathbf 1}$
on this span.
Define the candidate hitting-time function $\widetilde V^\pi(x,g)
:=
\langle \phi(g)-\phi(x),\omega_g^\pi\rangle_{\mathcal H}$.
For $x\neq g$ we have
\begin{align*}
(I-P^\pi)\widetilde V^\pi(\cdot,g)(x)
&=
-\langle \phi(x),\omega_g^\pi\rangle_{\mathcal H}
+
\mathbb E_{x'\sim \mathcal P^\pi(\cdot\mid x)}
\bigl[
\langle \phi(x'),\omega_g^\pi\rangle_{\mathcal H}
\bigr] \\
&=
\langle (\mathcal K_\pi-I)\phi(x),\omega_g^\pi\rangle_{\mathcal H} =
\langle \phi(x),(\mathcal K_\pi-I)^*\omega_g^\pi\rangle_{\mathcal H} =
\langle \phi(x),\mathbf w_{\mathbf 1}\rangle_{\mathcal H}
=
1.
\end{align*}
Moreover, $\widetilde V^\pi(g,g)=0$. Let $T_g^\pi$ be the hitting time of $g$ and define the stopping times $\tau_n:=T_g^\pi\wedge n$.
Consider $M_n
:=
\widetilde V^\pi(X_{\tau_n},g)+\tau_n$.
Because $\widetilde V^\pi$ satisfies the Poisson equation on
$\mathcal X\setminus\{g\}$ and satisfies the boundary condition
$\widetilde V^\pi(g,g)=0$, the stopped process $(M_n)_{n\ge 0}$ is a
martingale. Hence
\[
\widetilde V^\pi(x,g)
=
\mathbb E_x[M_n]
=
\mathbb E_x[\tau_n]
+
\mathbb E_x[\widetilde V^\pi(X_{\tau_n},g)].
\]

Since $\phi$ is bounded, there exists $C_\phi<\infty$ such that
$\sup_{x\in\mathcal X}\|\phi(x)\|_{\mathcal H}\le C_\phi$.
Therefore the candidate is bounded:
\[
|\widetilde V^\pi(x,g)|
\le
\|\phi(g)-\phi(x)\|_{\mathcal H}\|\omega_g^\pi\|_{\mathcal H}
\le
2C_\phi\|\omega_g^\pi\|_{\mathcal H}.
\]
Let $C_g:=2C_\phi\|\omega_g^\pi\|_{\mathcal H}$ then $\bigl|\widetilde V^\pi(X_{\tau_n},g)\bigr|
\le C_g$,
and the martingale identity implies
\[
\mathbb E_x[\tau_n]
=
\widetilde V^\pi(x,g)
-
\mathbb E_x[\widetilde V^\pi(X_{\tau_n},g)]
\le
|\widetilde V^\pi(x,g)|+C_g
\le
2C_g.
\]
Thus $\sup_n \mathbb E_x[T_g^\pi\wedge n]<\infty$.
By monotone convergence we get $\mathbb E_x[T_g^\pi]
=
\lim_{n\to\infty}\mathbb E_x[T_g^\pi\wedge n]
\le
2C_g
<
\infty$.In particular, $T_g^\pi<\infty$ almost surely under the trajectory law
induced by $\pi$ from initial state $x$. Hence
$\mathbb I[T_g^\pi>n]\to 0$ almost surely under this same law. Furthermore, $\widetilde V^\pi(X_{\tau_n},g)
=
\widetilde V^\pi(X_n,g)\mathbb I[T_g^\pi>n]$,
because if $T_g^\pi\le n$, then $\tau_n=T_g^\pi$ and
$X_{\tau_n}=g$, so the boundary condition gives
$\widetilde V^\pi(X_{\tau_n},g)=0$.
Since $\widetilde V^\pi$ is bounded, dominated convergence yields $\mathbb E_x[\widetilde V^\pi(X_{\tau_n},g)]\to 0$.
Also, $\tau_n\uparrow T_g^\pi$, so by monotone convergence, $\mathbb E_x[\tau_n]\to \mathbb E_x[T_g^\pi]=V^\pi(x,g)$.
Letting $n\to\infty$ in $\widetilde V^\pi(x,g)
=
\mathbb E_x[\tau_n]
+
\mathbb E_x[\widetilde V^\pi(X_{\tau_n},g)]$
gives $\widetilde V^\pi(x,g)=V^\pi(x,g)$.
Therefore $V^\pi(x,g)
=
\langle \phi(g)-\phi(x),\omega_g^\pi\rangle_{\mathcal H}$.
\end{proof}

\subsection{Proof of Theorem \ref{thm:identifiability_of_linear_hitting_time_representation}}
\begin{proof}
Let $d_g:=\phi(g)-\phi(x_0)\in \mathcal D$ and $d_g^\star:=\phi^\star(g)-\phi^\star(x_0)\in \mathcal D^\star$.
We use the separating policy $\pi_{\mathrm{sep}}$ from the assumptions. First, fix $g,h\in \mathcal X$. By definition of the analysis operators,
\[
(\mathcal A_{\pi_{\mathrm{sep}}} d_g)(h)
=
\langle \phi(g)-\phi(x_0),\, \omega_h^{\pi_{\mathrm{sep}}}\rangle_{\mathcal{H}}.
\]
Using the hitting-time representation for the alternative representation,
\[
\langle \phi(g)-\phi(x_0),\, \omega_h^{\pi_{\mathrm{sep}}}\rangle_{\mathcal H}
=
V^{\pi_{\mathrm{sep}}}(x_0,h)-V^{\pi_{\mathrm{sep}}}(g,h).
\]
Applying the same identity in the canonical representation yields
\[
V^{\pi_{\mathrm{sep}}}(x_0,h)-V^{\pi_{\mathrm{sep}}}(g,h)
=
\langle \phi^\star(g)-\phi^\star(x_0),\, \omega_h^{\star,\pi_{\mathrm{sep}}}\rangle_{\mathcal H^\star}
=
(\mathcal B_{\pi_{\mathrm{sep}}} d_g^\star)(h).
\]
Hence $\mathcal A_{\pi_{\mathrm{sep}}} d_g
=
\mathcal B_{\pi_{\mathrm{sep}}} d_g^\star$ for all $g\in \mathcal X$. Since $\mathcal D=\overline{\mathrm{span}}\{d_g:g\in \mathcal X\}$ and
$\mathcal D^\star=\overline{\mathrm{span}}\{d_g^\star:g\in \mathcal X\}$, and both
$\mathcal A_{\pi_{\mathrm{sep}}}$ and $\mathcal B_{\pi_{\mathrm{sep}}}$ are bounded,
their images are the closures of the spans of the images of these generators: $\mathcal A_{\pi_{\mathrm{sep}}}(\mathcal D)
=
\overline{\mathrm{span}}\{\mathcal A_{\pi_{\mathrm{sep}}} d_g:g\in \mathcal X\}$ and $\mathcal B_{\pi_{\mathrm{sep}}}(\mathcal D^\star)
=
\overline{\mathrm{span}}\{\mathcal B_{\pi_{\mathrm{sep}}} d_g^\star:g\in \mathcal X\}$.
As the generator images coincide, the ranges coincide we get 
$\mathcal A_{\pi_{\mathrm{sep}}}(\mathcal D)
=
\mathcal B_{\pi_{\mathrm{sep}}}(\mathcal D^\star)$. Now define $M:=\mathcal B_{\pi_{\mathrm{sep}}}^{-1}\mathcal A_{\pi_{\mathrm{sep}}}:\mathcal D\to \mathcal D^\star$. This is well-defined because the ranges coincide. Since
$\mathcal A_{\pi_{\mathrm{sep}}}$ and $\mathcal B_{\pi_{\mathrm{sep}}}$ are bounded and bounded-below,
each is injective and has bounded inverse on its range. Therefore $M$ is a bounded linear isomorphism,
with inverse $M^{-1}=\mathcal A_{\pi_{\mathrm{sep}}}^{-1}\mathcal B_{\pi_{\mathrm{sep}}}$. Moreover, for every $g\in \mathcal X$, we have
$M d_g
=
\mathcal B_{\pi_{\mathrm{sep}}}^{-1}\mathcal A_{\pi_{\mathrm{sep}}} d_g
=
\mathcal B_{\pi_{\mathrm{sep}}}^{-1}\mathcal B_{\pi_{\mathrm{sep}}} d_g^\star
=
d_g^\star$.
Thus $M(\phi(g)-\phi(x_0))=\phi^\star(g)-\phi^\star(x_0)$ for all $g\in \mathcal X$. Now let $x,g\in \mathcal X$ be arbitrary. Since
\[
\phi(g)-\phi(x)
=
(\phi(g)-\phi(x_0))-(\phi(x)-\phi(x_0)),
\]
linearity of $M$ gives
\[
M(\phi(g)-\phi(x))
=
(\phi^\star(g)-\phi^\star(x_0))-(\phi^\star(x)-\phi^\star(x_0))
=
\phi^\star(g)-\phi^\star(x).
\]
This proves the displacement-geometry identity. To show uniqueness, suppose $\widetilde M:\mathcal D\to \mathcal D^\star$ is another bounded linear map such that $\widetilde M(\phi(g)-\phi(x_0))=\phi^\star(g)-\phi^\star(x_0)$ for all $g\in \mathcal X$. Then $M$ and $\widetilde M$ agree on the dense generating set $\mathrm{span}\{\phi(g)-\phi(x_0):g\in \mathcal X\}$,
and hence, by continuity, on all of $\mathcal D$. Therefore $M$ is unique.

Next, fix any admissible policy $\pi\in\Pi_D$ for which $\mathcal K_\pi$ and $\mathcal K_\pi^\star$ are well-defined.
Since $(\mathcal X, \mathcal A, \mathcal P)$ is a Sufficient CMP, for every $x\in \mathcal X$ and $a\in \mathcal A$, we have
$\mathcal K_a\phi(x)=\mathbb E_{x'\sim \mathcal P(\cdot\mid x,a)}[\phi(x')]$ and $
\mathcal K_a^\star\phi^\star(x)=\mathbb E_{x'\sim \mathcal P(\cdot\mid x,a)}[\phi^\star(x')]$.
Averaging over $a\sim \pi(\cdot\mid x)$ gives $\mathcal K_\pi\phi(x)=\mathbb E_{x'\sim \mathcal P^\pi(\cdot\mid x)}[\phi(x')]$ and
$\mathcal K_\pi^\star\phi^\star(x)=\mathbb E_{x'\sim \mathcal P^\pi(\cdot\mid x)}[\phi^\star(x')]$.
Hence $\mathcal K_\pi\phi(x)-\phi(x_0)
=
\mathbb E_{x'\sim \mathcal P^\pi(\cdot\mid x)}[\phi(x')-\phi(x_0)]$. Applying the bounded linear map $M$ and passing through the Bochner expectation,
\[
M\bigl(\mathcal K_\pi\phi(x)-\phi(x_0)\bigr)
=
\mathbb E_{x'\sim \mathcal P^\pi(\cdot\mid x)}
\bigl[M(\phi(x')-\phi(x_0))\bigr].
\]
Using the displacement identity already proved, $M(\phi(x')-\phi(x_0))=\phi^\star(x')-\phi^\star(x_0)$,
we obtain
\[
M\bigl(\mathcal K_\pi\phi(x)-\phi(x_0)\bigr)
=
\mathbb E_{x'\sim \mathcal P^\pi(\cdot\mid x)}[\phi^\star(x')-\phi^\star(x_0)]
=
\mathcal K_\pi^\star\phi^\star(x)-\phi^\star(x_0).
\]
Subtracting this identity from $M(\phi(g)-\phi(x_0))=\phi^\star(g)-\phi^\star(x_0)$
yields
\[
M\bigl(\phi(g)- \mathcal K_\pi\phi(x)\bigr)=\phi^\star(g)-\mathcal K_\pi^\star\phi^\star(x),
\qquad \forall x,g\in \mathcal X.
\]
Finally, fix any action $a\in \mathcal A$. By exact sufficiency again,
\[
\mathcal K_a\phi(x)-\phi(x_0)
=
\mathbb E_{x'\sim \mathcal P(\cdot\mid x,a)}[\phi(x')-\phi(x_0)].
\]
Applying $M$ and using the displacement identity inside the expectation,
\[
M\bigl(\mathcal K_a\phi(x)-\phi(x_0)\bigr)
=
\mathbb E_{x'\sim \mathcal P(\cdot\mid x,a)}[\phi^\star(x')-\phi^\star(x_0)]
=
\mathcal K_a^\star\phi^\star(x)-\phi^\star(x_0).
\]
Subtracting from $M(\phi(g)-\phi(x_0))=\phi^\star(g)-\phi^\star(x_0)$
gives
\[
M\bigl(\phi(g)- \mathcal K_a\phi(x)\bigr)=\phi^\star(g)-\mathcal K_a^\star\phi^\star(x),
\qquad \forall x,g\in \mathcal X,\ \forall a\in \mathcal A.
\]
This completes the proof.
\end{proof}

\subsection{Proof of Corollary \ref{cor:hitting_time_equivalence}}
\begin{proof}
Fix a policy $\pi$ admitting canonical representers $\omega_g^{\star,\pi}$, and fix $g\in \mathcal X$.
Let $P_{\mathcal D^\star}: \mathcal H^\star \to \mathcal D^\star$ denote the orthogonal projection onto $\mathcal D^\star$.
Define $\ell_g^\pi(d):=\langle M d,\, P_{\mathcal D^\star}\omega_g^{\star,\pi}\rangle_{\mathcal H^\star}$ for $d\in \mathcal D$. Since $M$ and $P_{\mathcal D^\star}$ are bounded, $\ell_g^\pi$ is a bounded linear functional on $\mathcal D$.
By the Riesz representation theorem, there exists a unique $\bar \omega_g^\pi\in \mathcal D$ such that $\ell_g^\pi(d)=\langle d,\, \bar \omega_g^\pi\rangle_{\mathcal H}$ for all $d\in \mathcal D$. Now fix $x\in \mathcal X$. By
Theorem~\ref{thm:identifiability_of_linear_hitting_time_representation}, we get $M(\phi(g)-\phi(x))=\phi^\star(g)-\phi^\star(x)$.
Since $\phi^\star(g)-\phi^\star(x)\in \mathcal D^\star$, projection does not change the inner product: $\langle \phi^\star(g)-\phi^\star(x),\, \omega_g^{\star,\pi}\rangle_{\mathcal H^\star}
=
\langle \phi^\star(g)-\phi^\star(x),\, P_{\mathcal D^\star}\omega_g^{\star,\pi}\rangle_{\mathcal H^\star}$.
Therefore
\begin{align*}
\langle \phi^\star(g)-\phi^\star(x),\, \omega_g^{\star,\pi}\rangle_{\mathcal H^\star}
&=
\langle M(\phi(g)-\phi(x)),\, P_{\mathcal D^\star}\omega_g^{\star,\pi}\rangle_{\mathcal H^\star}
=
\ell_g^\pi(\phi(g)-\phi(x))\\
&=
\langle \phi(g)-\phi(x),\, \bar \omega_g^\pi\rangle_{\mathcal H}.    
\end{align*}
Since the left-hand side equals $V^\pi(x,g)$ by assumption, we obtain
\[
V^\pi(x,g)=\langle \phi(g)-\phi(x),\, \bar \omega_g^\pi\rangle_{\mathcal H},
\qquad \forall x,g\in \mathcal X.
\]

Now suppose $\pi\in\Pi_D$ is admissible and that $\mathcal K_\pi, \mathcal K_\pi^\star$ are well-defined.
By Theorem~\ref{thm:identifiability_of_linear_hitting_time_representation},
\[
M(\phi(g)- \mathcal K_\pi\phi(x))=\phi^\star(g)- \mathcal K_\pi^\star\phi^\star(x).
\]
Since $\phi^\star(g)- \mathcal K_\pi^\star\phi^\star(x)\in \mathcal D^\star$,
\begin{align*}
\langle \phi^\star(g)- \mathcal K_\pi^\star\phi^\star(x),\, \omega_g^{\star,\pi}\rangle_{\mathcal H^\star}
&=
\langle \phi^\star(g)- \mathcal K_\pi^\star\phi^\star(x),\, P_{\mathcal D^\star}\omega_g^{\star,\pi}\rangle_{\mathcal H^\star}
=
\ell_g^\pi(\phi(g)- \mathcal K_\pi\phi(x))\\
&=
\langle \phi(g)- \mathcal K_\pi\phi(x),\, \bar \omega_g^\pi\rangle_{\mathcal H}.    
\end{align*}
Finally, for any action $a\in \mathcal A$, Theorem~\ref{thm:identifiability_of_linear_hitting_time_representation} gives
\[
M(\phi(g)- \mathcal K_a\phi(x))=\phi^\star(g)- \mathcal K_a^\star\phi^\star(x).
\]
Again using that $\phi^\star(g)- \mathcal K_a^\star\phi^\star(x)\in \mathcal D^\star$,
\begin{align*}
\langle \phi^\star(g)- \mathcal K_a^\star\phi^\star(x),\, \omega_g^{\star,\pi}\rangle_{\mathcal H^\star}
&=
\langle \phi^\star(g)- \mathcal K_a^\star\phi^\star(x),\, P_{\mathcal D^\star}\omega_g^{\star,\pi}\rangle_{\mathcal H^\star}
=
\ell_g^\pi(\phi(g)- \mathcal K_a\phi(x))\\
&=
\langle \phi(g)- \mathcal K_a\phi(x),\, \bar \omega_g^\pi\rangle_{\mathcal H}.    
\end{align*}
This proves the result.
\end{proof}

\subsection{Proof of Lemma \ref{lem:consistency_of_sufficient_capacity}}

\begin{proof}
Fix $x\in\mathcal X$. Since the compressed encoder is defined by $\phi(x')=P_d\widetilde\phi(x')$,
we have, by linearity and boundedness of the orthogonal projection $P_d$,
\[
\mathbb E_{x'\sim \mathcal P^\pi(\cdot\mid x)}[\phi(x')]
=
\mathbb E_{x'\sim \mathcal P^\pi(\cdot\mid x)}[P_d\widetilde\phi(x')]
=
P_d
\mathbb E_{x'\sim \mathcal P^\pi(\cdot\mid x)}[\widetilde\phi(x')].
\]
By definition of the ambient feature transition operator, $\mathbb E_{x'\sim \mathcal P^\pi(\cdot\mid x)}[\widetilde\phi(x')]
=
\widetilde{\mathcal T}^{\pi}\widetilde\phi(x)$.
Therefore $\mathbb E_{x'\sim \mathcal P^\pi(\cdot\mid x)}[\phi(x')]
=
P_d\widetilde{\mathcal T}^{\pi}\widetilde\phi(x)$.
The sufficient capacity condition states  $\bigl\|
P_d\widetilde{\mathcal T}^{\pi}\widetilde\phi(x)
-
\mathcal T^\pi\phi(x)
\bigr\|_{\mathcal H}
\le
\epsilon$.
Combining the previous identity with this bound gives
\[
\bigl\|
\mathbb E_{x'\sim \mathcal P^\pi(\cdot\mid x)}[\phi(x')]
-
\mathcal T^\pi\phi(x)
\bigr\|_{\mathcal H}
\le
\epsilon.
\]
Since $x\in\mathcal X$ was arbitrary, the claim follows.
\end{proof}

\subsection{Proof of Theorem \ref{thm:approximation_error}}
\begin{proof}
Fix $g$ and define the projected feature
$z(x):=P_Q\phi(x)\in\mathcal H_Q$.
Since $\omega_g^\pi\in\mathcal H_Q$, we have
\[
\widehat V_g^\pi(x)
=
\langle \phi(g)-\phi(x),\omega_g^\pi\rangle_{\mathcal H}
=
-\langle z(x),\omega_g^\pi\rangle_{\mathcal H},
\]
because $\phi(g)\in\mathcal H_g$ and $\omega_g^\pi\perp\mathcal H_g$. On the transient set $\mathcal X\setminus\{g\}$, the true hitting time
satisfies the scalar Bellman equation
$V^\pi(\cdot,g)
=
1+P_Q^\pi V^\pi(\cdot,g)$,
where $P_Q^\pi$ is the killed scalar transition operator and the boundary value at the goal is $V^\pi(g,g)=0$.

Define the prediction error $e_g(x):=\widehat V_g^\pi(x)-V^\pi(x,g)$
and the scalar Bellman residual $r_g(x)
:=
\widehat V_g^\pi(x)
-
\bigl(1+P_Q^\pi\widehat V_g^\pi(x)\bigr)$.
Then $(I-P_Q^\pi)e_g=r_g$. We now bound $r_g$. Since $\widehat V_g^\pi(g)=0$, applying the killed
operator to $\widehat V_g^\pi$ is equivalent to taking the usual one-step
expectation with the zero boundary at $g$. Thus, for
$x\in\mathcal X\setminus\{g\}$,
\[
\begin{aligned}
r_g(x)
&=
-\langle z(x),\omega_g^\pi\rangle_{\mathcal H}
-1
+
\mathbb E_{x'\sim \mathcal P^\pi(\cdot\mid x)}
\bigl[
\langle z(x'),\omega_g^\pi\rangle_{\mathcal H}
\bigr] \\
&=
-\langle z(x),\omega_g^\pi\rangle_{\mathcal H}
-1
+
\left\langle
\mathbb E_{x'\sim \mathcal P^\pi(\cdot\mid x)}[z(x')],
\omega_g^\pi
\right\rangle_{\mathcal H}.
\end{aligned}
\]
Here the equality is consistent with the killed operator because
$z(g)=P_Q\phi(g)=0$.

By Lemma~\ref{lem:consistency_of_sufficient_capacity}, for each
$x\in\mathcal X$ there exists $\xi(x)\in\mathcal H$ with $\|\xi(x)\|_{\mathcal H}\le \epsilon$
such that $\mathbb E_{x'\sim \mathcal P^\pi(\cdot\mid x)}[\phi(x')]
=
\mathcal T^\pi\phi(x)+\xi(x)$.
Projecting onto $\mathcal H_Q$ gives
\[
\mathbb E_{x'\sim \mathcal P^\pi(\cdot\mid x)}[z(x')]
=
P_Q\mathcal T^\pi\phi(x)+P_Q\xi(x).
\]
Now decompose $\phi(x)=P_Q\phi(x)+P_g\phi(x)=z(x)+P_g\phi(x)$,
where $P_g:=I-P_Q$ is the orthogonal projection onto $\mathcal H_g$.
Since the goal subspace is invariant under $\mathcal T^\pi$, we have
$\mathcal T^\pi(\mathcal H_g)\subseteq\mathcal H_g$.
Hence $P_Q\mathcal T^\pi P_g\phi(x)=0$.
Therefore $P_Q\mathcal T^\pi\phi(x)
=
P_Q\mathcal T^\pi z(x)
=
\mathcal T_Q^\pi z(x)$.
Thus $\mathbb E_{x'\sim \mathcal P^\pi(\cdot\mid x)}[z(x')]
=
\mathcal T_Q^\pi z(x)+P_Q\xi(x)$.
Substituting this identity into the residual gives
\[
\begin{aligned}
r_g(x)
&=
-\langle z(x),\omega_g^\pi\rangle_{\mathcal H}
-1
+
\langle \mathcal T_Q^\pi z(x)+P_Q\xi(x),\omega_g^\pi\rangle_{\mathcal H} \\
&=
\langle (\mathcal T_Q^\pi-I)z(x),\omega_g^\pi\rangle_{\mathcal H}
+
\langle P_Q\xi(x),\omega_g^\pi\rangle_{\mathcal H}
-1.
\end{aligned}
\]
Using the adjoint equation $(\mathcal T_Q^\pi-I)^*\omega_g^\pi=\mathbf w_{\mathbf 1}$,
we obtain $r_g(x)
=
\langle z(x),\mathbf w_{\mathbf 1}\rangle_{\mathcal H}
+
\langle P_Q\xi(x),\omega_g^\pi\rangle_{\mathcal H}
-1$.
Because $\mathbf w_{\mathbf 1}\in\mathcal H_Q$, we have $\langle z(x),\mathbf w_{\mathbf 1}\rangle_{\mathcal H}
=
\langle P_Q\phi(x),\mathbf w_{\mathbf 1}\rangle_{\mathcal H}
=
\langle \phi(x),\mathbf w_{\mathbf 1}\rangle_{\mathcal H}
=
1$.
Therefore $ r_g(x)
=
\langle P_Q\xi(x),\omega_g^\pi\rangle_{\mathcal H}$.
Hence $|r_g(x)|
\le
\|P_Q\xi(x)\|_{\mathcal H}\|\omega_g^\pi\|_{\mathcal H}
\le
\|\xi(x)\|_{\mathcal H}\|\omega_g^\pi\|_{\mathcal H}
\le
\epsilon\|\omega_g^\pi\|_{\mathcal H}$.
Taking the supremum over $x\in\mathcal X\setminus\{g\}$ yields $\|r_g\|_\infty
\le
\epsilon\|\omega_g^\pi\|_{\mathcal H}$.
Since $(I-P_Q^\pi)e_g=r_g$,
we have $e_g=(I-P_Q^\pi)^{-1}r_g$.
Therefore, by the assumed resolvent bound,
\[
\|e_g\|_\infty
\le
\bigl\|(I-P_Q^\pi)^{-1}\bigr\|_\infty
\|r_g\|_\infty
\le
\frac{C_H}{1-\rho(\mathcal T_Q^\pi)}
\epsilon\|\omega_g^\pi\|_{\mathcal H}.
\]
Since $e_g(x)=\widehat V_g^\pi(x)-V^\pi(x,g)$
we conclude that
\[
\sup_{x\in\mathcal X\setminus\{g\}}
\bigl|V^\pi(x,g)-\widehat V_g^\pi(x)\bigr|
\le
\frac{C_H\,\|\omega_g^\pi\|_{\mathcal H}\,\epsilon}
{1-\rho(\mathcal T_Q^\pi)}.
\]
\end{proof}

\subsection{Proof of Theorem \ref{thm:oracle_displacement_recovery}}
\begin{proof}
Fix $\delta>0$ and choose any $r_\delta\in\mathfrak R_\epsilon$ satisfying
\[
\frac{C_H(r_\delta)\,C_\omega(r_\delta)\epsilon}{1-\rho_{r_\delta}}
\le
\eta_\epsilon^\star+\delta.
\]
For this reference tuple, define the reference hitting-time predictor
\[
V_{r_\delta}(x,g)
:=
\left\langle
\phi^{r_\delta}(g)-\phi^{r_\delta}(x),
\omega_g^{r_\delta,\pi_{\mathrm{sep}}}
\right\rangle_{\mathcal H^{r_\delta}}.
\]
By Theorem~\ref{thm:approximation_error}, for each fixed goal $g$,
\[
\sup_{x\in\mathcal X\setminus\{g\}}
\left|
V_{r_\delta}(x,g)-V^{\pi_{\mathrm{sep}}}(x,g)
\right|
\le
\frac{
C_H(r_\delta,g)
\|\omega_g^{r_\delta,\pi_{\mathrm{sep}}}\|_{\mathcal H^{r_\delta}}
\epsilon
}{
1-\rho(\mathcal T_{Q,g}^{r_\delta,\pi_{\mathrm{sep}}})
}.
\]
Since both sides vanish at $x=g$, the same bound holds for all
$x\in\mathcal X$. Taking suprema over $g$ gives
\[
\sup_{x,g\in\mathcal X}
\left|
V_{r_\delta}(x,g)-V^{\pi_{\mathrm{sep}}}(x,g)
\right|
\le
\frac{
C_H(r_\delta)C_\omega(r_\delta)\epsilon
}{
1-\rho_{r_\delta}
}
\le
\eta_\epsilon^\star+\delta.
\]
Because $\nu$ is a probability measure, this implies $\|V_{r_\delta}-V^{\pi_{\mathrm{sep}}}\|_{L^2(\nu\times\nu)}
\le
\eta_\epsilon^\star+\delta$. Since $\widehat V$ minimizes $\mathcal L$ over $\mathcal F$, and since $\inf_{f\in\mathcal F}
\|f-V^{\pi_{\mathrm{sep}}}\|_{L^2(\nu\times\nu)}
\le
\eta_\epsilon^\star+\eta_{\mathrm{cls}}$
we have $\|\widehat V-V^{\pi_{\mathrm{sep}}}\|_{L^2(\nu\times\nu)}
\le
\eta_\epsilon^\star+\eta_{\mathrm{cls}}$. Define $\Delta(x,g):=\widehat V(x,g)-V_{r_\delta}(x,g)$.
By the triangle inequality,
\[
\|\Delta\|_{L^2(\nu\times\nu)}
\le
\|\widehat V-V^{\pi_{\mathrm{sep}}}\|_{L^2(\nu\times\nu)}
+
\|V_{r_\delta}-V^{\pi_{\mathrm{sep}}}\|_{L^2(\nu\times\nu)}
\le
2\eta_\epsilon^\star+\eta_{\mathrm{cls}}+\delta.
\]

For each $(x,g)\in\mathcal X\times\mathcal X$, define
$d_{x,g}:=\phi(g)-\phi(x)\in D$ and $d_{x,g}^{r_\delta}
:=
\phi^{r_\delta}(g)-\phi^{r_\delta}(x)\in D_{r_\delta}$.
For any $h\in\mathcal X$, by definition of the learned and reference analysis
operators, $(A d_{x,g})(h)
=
\widehat V(x,h)-\widehat V(g,h)$
and $(B_{r_\delta} d_{x,g}^{r_\delta})(h)
=
V_{r_\delta}(x,h)-V_{r_\delta}(g,h)$.
Therefore $(A d_{x,g})(h)
-
(B_{r_\delta}d_{x,g}^{r_\delta})(h)
=
\Delta(x,h)-\Delta(g,h)$.
Hence
\begin{align*}
\int_{\mathcal X}\int_{\mathcal X}
\|A d_{x,g}-B_{r_\delta}d_{x,g}^{r_\delta}\|_{L^2(\nu)}^2
\,d\nu(x)\,d\nu(g)
 =
\int_{\mathcal X}\int_{\mathcal X}\int_{\mathcal X}
|\Delta(x,h)-\Delta(g,h)|^2
\,d\nu(h)\,d\nu(x)\,d\nu(g).
\end{align*}
Using $|\Delta(x,h)-\Delta(g,h)|^2
\le
2|\Delta(x,h)|^2+2|\Delta(g,h)|^2$and the fact that $\nu$ is a probability measure, we obtain
\[
\int_{\mathcal X}\int_{\mathcal X}
\|A d_{x,g}-B_{r_\delta}d_{x,g}^{r_\delta}\|_{L^2(\nu)}^2
\,d\nu(x)\,d\nu(g)
\le
4\|\Delta\|_{L^2(\nu\times\nu)}^2.
\]
Therefore
\[
\left(
\int_{\mathcal X}\int_{\mathcal X}
\|A d_{x,g}-B_{r_\delta}d_{x,g}^{r_\delta}\|_{L^2(\nu)}^2
\,d\nu(x)\,d\nu(g)
\right)^{1/2}
\le
2\bigl(2\eta_\epsilon^\star+\eta_{\mathrm{cls}}+\delta\bigr).
\]

Since $B_{r_\delta}$ is bounded below, it is injective, has closed range, and
its inverse on $\mathcal R(B_{r_\delta})$ satisfies
\[
\|B_{r_\delta}^{-1}f\|_{\mathcal H^{r_\delta}}
\le
\frac{1}{m_{B,\delta}}\|f\|_{L^2(\nu)},
\qquad
f\in\mathcal R(B_{r_\delta}).
\]
Let $P_{\mathcal R(B_{r_\delta})}$ denote the orthogonal projection in
$L^2(\nu)$ onto $\mathcal R(B_{r_\delta})$, and define
$M_\delta
:=
B_{r_\delta}^{-1}P_{\mathcal R(B_{r_\delta})}A$.
Then $M_\delta:D\to D_{r_\delta}$ is a bounded linear map. Moreover, since
$P_{\mathcal R(B_{r_\delta})}$ is contractive,
\[
\begin{aligned}
\|M_\delta d_{x,g}-d_{x,g}^{r_\delta}\|_{\mathcal H^{r_\delta}}
&=
\left\|
B_{r_\delta}^{-1}
P_{\mathcal R(B_{r_\delta})}
\left(
A d_{x,g}-B_{r_\delta}d_{x,g}^{r_\delta}
\right)
\right\|_{\mathcal H^{r_\delta}}
\\
&\le
\frac{1}{m_{B,\delta}}
\|A d_{x,g}-B_{r_\delta}d_{x,g}^{r_\delta}\|_{L^2(\nu)}.
\end{aligned}
\]
Squaring, integrating over $(x,g)$, and taking square roots yields
\[
\left(
\int_{\mathcal X}\int_{\mathcal X}
\|M_\delta d_{x,g}-d_{x,g}^{r_\delta}\|_{\mathcal H^{r_\delta}}^2
\,d\nu(x)\,d\nu(g)
\right)^{1/2}
\le
\frac{2}{m_{B,\delta}}
\bigl(2\eta_\epsilon^\star+\eta_{\mathrm{cls}}+\delta\bigr).
\]
Substituting $d_{x,g}=\phi(g)-\phi(x)$ and $d_{x,g}^{r_\delta}
=
\phi^{r_\delta}(g)-\phi^{r_\delta}(x)$
gives the claimed bound.
\end{proof}
\subsection{Proof of Theorem \ref{thm:pac_oracle_displacement_topology}}
\label{app:finite_sample_oracle_recovery}
This subsection proves Theorem~\ref{thm:pac_oracle_displacement_topology}. The proof has two steps.
We first establish a finite-sample excess-risk bound for the phase-2 tuple-regression problem induced
by the hitting-time channel of Algorithm~1, conditional on a fixed frozen task encoder $\bar \omega_\psi$.
We then lift this statistical control to a displacement-topology recovery statement by comparing the
learned predictor to a near-oracle reference representation from $\mathfrak R_\epsilon$, using the
coverage assumption on the conditional goal law and the bounded-below property of the reference
analysis operator.

\begin{lemma}[Trajectory-level finite-sample control for replay-buffer phase-2 regression]
\label{lem:trajectory_level_phase2_tuple_regression}
Assume $\pi_{\mathrm{sep}}=\pi_B$, and let $D_N=\{\tau^1,\dots,\tau^N\}$
be $N$ i.i.d. trajectories generated by $\pi_B$. Fix a task encoder $\bar \omega_\psi:\mathcal X\to \mathcal H$ satisfying $\|\bar \omega_\psi(g)\|_{\mathcal H}\le C_{\bar \omega_\psi}$ for all $g\in \mathcal X$
and let $\Phi$ be a class of state encoders satisfying $\|\phi(x)\|_{\mathcal H}\le C_\phi$ for all $\phi\in\Phi,\ x\in \mathcal X$.
Let $\operatorname{Ext}(\tau)=\bigl(z_1(\tau),\dots,z_{m_{\mathrm{tr}}}(\tau)\bigr)$ and
$z_t(\tau)=\bigl(s_t(\tau),u_t(\tau),g_t(\tau),H_t(\tau)\bigr)$
be a measurable tuple-extraction rule that returns exactly $m_{\mathrm{tr}}$ tuples from each
trajectory. Let $\mu$ denote the law of a uniformly selected extracted tuple, i.e. the law obtained by
first sampling $\tau\sim\pi_B$, then $T\sim\mathrm{Unif}\{1,\dots,m_{\mathrm{tr}}\}$ independently,
and setting $(S,U,\mathsf G,H):=z_T(\tau)$. Define the replay-buffer conditional mean
$m_{\mathrm{traj}}(s,u,g):=\mathbb E[H\mid S=s,U=u,\mathsf G=g]$
and the phase-2 predictor class
$\mathcal F_{\Phi,\bar \omega_\psi}
:=
\left\{
f_\phi(s,u,g)=\langle \phi(u)-\phi(s),\bar \omega_\psi(g)\rangle_{\mathcal H}
:
\phi\in\Phi
\right\}$.
For each extraction slot $t\in\{1,\dots,m_{\mathrm{tr}}\}$, define the slotwise trajectory class
\[
(\mathcal F_{\Phi,\bar \omega_\psi})^{(t)}
:=
\left\{
\tau\mapsto f_\phi(s_t(\tau),u_t(\tau),g_t(\tau))
:
\phi\in\Phi
\right\},
\]
and write
\[
\overline{\mathfrak R}_N(\Phi,\bar \omega_\psi)
:=
\frac{1}{m_{\mathrm{tr}}}\sum_{t=1}^{m_{\mathrm{tr}}}
\mathfrak R_N\!\bigl((\mathcal F_{\Phi,\bar \omega_\psi})^{(t)}\bigr),
\]
where for any class $\mathcal G$ of real-valued functions on trajectories,
\[
\mathfrak R_N(\mathcal G)
:=
\mathbb E_{\tau_{1:N},\sigma}
\left[
\sup_{q \in\mathcal G}
\frac{1}{N}\sum_{j=1}^N \sigma_j q(\tau^j)
\right].
\]

Define the population and empirical risks by
\[
R(\phi):=\mathbb E_\mu\bigl[(f_\phi(S,U,\mathsf G)-H)^2\bigr],
\]
\[
\widehat R_N(\phi)
:=
\frac{1}{N m_{\mathrm{tr}}}
\sum_{j=1}^N\sum_{t=1}^{m_{\mathrm{tr}}}
\bigl(f_\phi(s_t(\tau^j),u_t(\tau^j),g_t(\tau^j))-H_t(\tau^j)\bigr)^2,
\]
and let $\widehat\phi_N\in\arg\min_{\phi\in\Phi}\widehat R_N(\phi)$.
Furthermore define
\[
\Gamma_N(\delta)
:=
16(2C_\phi C_{\bar \omega_\psi}+H_{\max})\,\overline{\mathfrak R}_N(\Phi,\bar \omega_\psi)
+
2(2C_\phi C_{\bar \omega_\psi}+H_{\max})^2
\sqrt{\frac{\log(1/\delta)}{2N}}.
\]
Then, with probability at least $1-\delta$,
\[
\|f_{\widehat\phi_N}-m_{\mathrm{traj}}\|_{L^2(\mu)}^2
\le
\inf_{\phi\in\Phi}\|f_\phi-m_{\mathrm{traj}}\|_{L^2(\mu)}^2
+
\Gamma_N(\delta).
\]
\end{lemma}

\begin{proof}
For each $\phi\in\Phi$, define the trajectory-level block loss
\[
\Psi_\phi(\tau)
:=
\frac{1}{m_{\mathrm{tr}}}
\sum_{t=1}^{m_{\mathrm{tr}}}
\bigl(f_\phi(s_t(\tau),u_t(\tau),g_t(\tau))-H_t(\tau)\bigr)^2.
\]
Then $\widehat R_N(\phi)=\frac{1}{N}\sum_{j=1}^N \Psi_\phi(\tau^j)$ and $R(\phi)=\mathbb E_\tau[\Psi_\phi(\tau)]$. We proceed in four steps. Firstly, we show uniform boundedness. For any tuple $(s,u,g,H)$ and any $\phi\in\Phi$,
\[
|f_\phi(s,u,g)|
=
|\langle \phi(u)-\phi(s),\bar \omega_\psi(g)\rangle_{\mathcal H}|
\le
\|\phi(u)-\phi(s)\|_{\mathcal H}\,\|\bar \omega_\psi(g)\|_{\mathcal H}
\le
2C_\phi C_{\bar \omega_\psi}.
\]
Hence $|f_\phi(s,u,g)-H|
\le
2C_\phi C_{\bar \omega_\psi}+H_{\max}$,
and therefore $0
\le
(f_\phi(s,u,g)-H)^2
\le
M_{\mathrm{loss}}$ and $
M_{\mathrm{loss}}:=(2C_\phi C_{\bar \omega_\psi}+H_{\max})^2$.
Averaging over $t$ gives $0\le \Psi_\phi(\tau)\le M_{\mathrm{loss}}$ for all $\phi\in\Phi,\ \tau$.

Secondly, we establish symmetrization and slotwise contraction.
Let $\mathcal L_{\Phi,\bar \omega_\psi}:=\{\Psi_\phi:\phi\in\Phi\}$ be the induced block-loss class. Since the independent objects are the trajectories
$\tau^1,\dots,\tau^N$, symmetrization gives
\[
\mathbb E\Bigg[\sup_{\phi\in\Phi}|R(\phi)-\widehat R_N(\phi)|\Bigg]
\le
4\,\mathfrak R_N(\mathcal L_{\Phi,\bar \omega_\psi}).
\] The additional factor of two comes from controlling both one-sided deviations
\(R(\phi)-\widehat R_N(\phi)\) and \(\widehat R_N(\phi)-R(\phi)\) using the
one-sided Rademacher complexity. Expanding the block loss and using subadditivity of the supremum,
\begin{align*}
\mathfrak R_N(\mathcal L_{\Phi,\bar \omega_\psi})
=
&\mathbb E_{\tau_{1:N},\sigma}
\Biggl[
\sup_{\phi\in\Phi}
\frac{1}{N}\sum_{j=1}^N \sigma_j \Psi_\phi(\tau^j)
\Biggr]\\
&\le
\frac{1}{m_{\mathrm{tr}}}
\sum_{t=1}^{m_{\mathrm{tr}}}
\mathbb E_{\tau_{1:N},\sigma}
\Biggl[
\sup_{\phi\in\Phi}
\frac{1}{N}\sum_{j=1}^N \sigma_j
\bigl(f_\phi(s_t(\tau^j),u_t(\tau^j),g_t(\tau^j))-H_t(\tau^j)\bigr)^2
\Biggr].    
\end{align*}

For fixed $t$ and $j$, define $\varphi_{j,t}(y):=(y-H_t(\tau^j))^2-H_t(\tau^j)^2$. Then $\varphi_{j,t}(0)=0$, and because $|y|\le 2C_\phi C_{\bar \omega_\psi}$ and
$|H_t(\tau^j)|\le H_{\max}$, we have $|\varphi_{j,t}'(y)|
=
|2(y-H_t(\tau^j))|
\le
2(2C_\phi C_{\bar \omega_\psi}+H_{\max})$. Moreover, the constant shift $-H_t(\tau^j)^2$ vanishes under the Rademacher average, so
By Lemma~\ref{lem:talagrand_squared_loss_contraction},\[
\mathfrak R_N(\mathcal L_{\Phi,\bar \omega_\psi})
\le
\frac{2(2C_\phi C_{\bar \omega_\psi}+H_{\max})}{m_{\mathrm{tr}}}
\sum_{t=1}^{m_{\mathrm{tr}}}
\mathfrak R_N\!\bigl((\mathcal F_{\Phi,\bar \omega_\psi})^{(t)}\bigr)
=
2(2C_\phi C_{\bar \omega_\psi}+H_{\max})\,
\overline{\mathfrak R}_N(\Phi,\bar \omega_\psi).
\]
Therefore
\[
\mathbb E\Bigl[\sup_{\phi\in\Phi}|R(\phi)-\widehat R_N(\phi)|\Bigr]
\le
8(2C_\phi C_{\bar \omega_\psi}+H_{\max})\,\overline{\mathfrak R}_N(\Phi,\bar \omega_\psi).
\]

Thirdly, we establish trajectory-level concentration. Define
\[
\Phi_N(\tau^1,\dots,\tau^N)
:=
\sup_{\phi\in\Phi}|R(\phi)-\widehat R_N(\phi)|.
\]
If one trajectory $\tau^j$ is replaced by another trajectory $\tau^{j\prime}$, then for every
$\phi\in\Phi$,
\[
\bigl|
\widehat R_N(\phi;\tau^1,\dots,\tau^N)
-
\widehat R_N(\phi;\tau^1,\dots,\tau^{j\prime},\dots,\tau^N)
\bigr|
\le
\frac{M_{\mathrm{loss}}}{N},
\]
because only one block-loss term changes and every block loss lies in $[0,M_{\mathrm{loss}}]$.
Hence
\[
|\Phi_N(\tau^1,\dots,\tau^N)-\Phi_N(\tau^1,\dots,\tau^{j\prime},\dots,\tau^N)|
\le
\frac{M_{\mathrm{loss}}}{N}.
\]
By Lemma~\ref{lem:trajectory_level_mcdiarmid}, applied to the trajectory-level
block-loss class $\mathcal L_{\Phi,\bar \omega_\psi}$ with
$B=M_{\mathrm{loss}}$, with probability at least $1-\delta$,
\[
\sup_{\phi\in\Phi}|R(\phi)-\widehat R_N(\phi)|
\le
8(2C_\phi C_{\bar \omega_\psi}+H_{\max})\,\overline{\mathfrak R}_N(\Phi,\bar \omega_\psi)
+
(2C_\phi C_{\bar \omega_\psi}+H_{\max})^2
\sqrt{\frac{\log(1/\delta)}{2N}}.
\]

In the last step, we perform ERM comparison and target identification. On the event above, for every $\phi\in\Phi$,
\[
R(\widehat\phi_N)
\le
\widehat R_N(\widehat\phi_N)
+
\sup_{\psi\in\Phi}|R(\psi)-\widehat R_N(\psi)|
\le
\widehat R_N(\phi)
+
\sup_{\psi\in\Phi}|R(\psi)-\widehat R_N(\psi)|.
\]
Applying the same deviation bound once more to $\widehat R_N(\phi)$ gives
\[
R(\widehat\phi_N)
\le
R(\phi)
+
2\sup_{\psi\in\Phi}|R(\psi)-\widehat R_N(\psi)|.
\]
Therefore, with probability at least $1-\delta$,
\[
R(\widehat\phi_N)
\le
R(\phi)
+
16(2C_\phi C_{\bar \omega_\psi}+H_{\max})\,\overline{\mathfrak R}_N(\Phi,\bar \omega_\psi)
+
2(2C_\phi C_{\bar \omega_\psi}+H_{\max})^2
\sqrt{\frac{\log(1/\delta)}{2N}}.
\]
Taking the infimum over $\phi\in\Phi$ proves $R(\widehat\phi_N)\le \inf_{\phi\in\Phi}R(\phi)+\Gamma_N(\delta)$.
It remains to identify the regression target induced by the square loss. By definition we have
$m_{\mathrm{traj}}(S,U,\mathsf G)=\mathbb E[H\mid S,U,\mathsf G]$. Fixing any $\phi\in\Phi$ and expanding the square yields $R(\phi)
=
\mathbb E_\mu\bigl[(f_\phi-m_{\mathrm{traj}}+m_{\mathrm{traj}}-H)^2\bigr]$.
Hence
\[
R(\phi)
=
\mathbb E_\mu\bigl[(f_\phi-m_{\mathrm{traj}})^2\bigr]
+
\mathbb E_\mu\bigl[(m_{\mathrm{traj}}-H)^2\bigr]
+
2\,\mathbb E_\mu\bigl[(f_\phi-m_{\mathrm{traj}})(m_{\mathrm{traj}}-H)\bigr].
\]
The cross term vanishes because
\[
\mathbb E[m_{\mathrm{traj}}-H\mid S,U,\mathsf G]
=
m_{\mathrm{traj}}(S,U,\mathsf G)-\mathbb E[H\mid S,U,\mathsf G]
=
0.
\]
Therefore, $R(\phi)
=
\|f_\phi-m_{\mathrm{traj}}\|_{L^2(\mu)}^2
+
\mathbb E_\mu\bigl[(m_{\mathrm{traj}}-H)^2\bigr]$. The second term is independent of $\phi$, so subtracting it from both sides of $R(\widehat\phi_N)\le \inf_{\phi\in\Phi}R(\phi)+\Gamma_N(\delta)$
yields 
\begin{align*}   
\|f_{\widehat\phi_N}-m_{\mathrm{traj}}\|_{L^2(\mu)}^2
\le
\inf_{\phi\in\Phi}\|f_\phi-m_{\mathrm{traj}}\|_{L^2(\mu)}^2
+
\Gamma_N(\delta), 
\end{align*} 
proving the claim.
\end{proof}

\begin{proof}[Proof of Theorem~\ref{thm:pac_oracle_displacement_topology}]
By Lemma~\ref{lem:trajectory_level_phase2_tuple_regression}, with probability at least
$1-\delta$,
\[
\|f_{\widehat\phi_N}-m_{\mathrm{traj}}\|_{L^2(\mu)}^2
\le
\inf_{\phi\in\Phi}\|f_\phi-m_{\mathrm{traj}}\|_{L^2(\mu)}^2
+
\Gamma_N(\delta).
\]
Using the trajectory-compatibility assumption and the triangle inequality,
\[
\inf_{\phi\in\Phi}\|f_\phi-m_{\mathrm{traj}}\|_{L^2(\mu)}
\le
\inf_{\phi\in\Phi}\|f_\phi-G^{\pi_B}\|_{L^2(\mu)}
+
\|G^{\pi_B}-m_{\mathrm{traj}}\|_{L^2(\mu)}
\le
2\eta_\epsilon^\star+\eta_{\mathrm{cls}}+\eta_{\mathrm{traj}}.
\]
Hence, it holds with probability at least $1-\delta$ that $\|f_{\widehat\phi_N}-m_{\mathrm{traj}}\|_{L^2(\mu)}^2
\le
(2\eta_\epsilon^\star+\eta_{\mathrm{cls}}+\eta_{\mathrm{traj}})^2
+
\Gamma_N(\delta)$.

Now compare the learned predictor to the ideal displacement target:
\[
\|f_{\widehat\phi_N}-G^{\pi_B}\|_{L^2(\mu)}^2
\le
2\|f_{\widehat\phi_N}-m_{\mathrm{traj}}\|_{L^2(\mu)}^2
+
2\|m_{\mathrm{traj}}-G^{\pi_B}\|_{L^2(\mu)}^2.
\]
Therefore, again with probability at least $1-\delta$,
\[
\|f_{\widehat\phi_N}-G^{\pi_B}\|_{L^2(\mu)}^2
\le
2(2\eta_\epsilon^\star+\eta_{\mathrm{cls}}+\eta_{\mathrm{traj}})^2
+
2\Gamma_N(\delta)
+
2\eta_{\mathrm{traj}}^2.
\]

Now fix $\delta_{\mathrm{or}}>0$ and choose
$r_{\delta_{\mathrm{or}}}\in\mathfrak R_\epsilon$ such that
\[
\frac{C_H(r_{\delta_{\mathrm{or}}})C_\omega(r_{\delta_{\mathrm{or}}})\epsilon}
{1-\rho_{r_{\delta_{\mathrm{or}}}}}
\le
\eta_\epsilon^\star+\delta_{\mathrm{or}}.
\]
Define the reference tuple predictor $G_{r_{\delta_{\mathrm{or}}}}(s,u,g)
:=
\left\langle
\phi^{r_{\delta_{\mathrm{or}}}}(u)-\phi^{r_{\delta_{\mathrm{or}}}}(s),
\omega_g^{\,r_{\delta_{\mathrm{or}}},\pi_B}
\right\rangle_{\mathcal H^{r_{\delta_{\mathrm{or}}}}}$. By Theorem~\ref{thm:approximation_error}, exactly as in the previous proof, $\|G_{r_{\delta_{\mathrm{or}}}}-G^{\pi_B}\|_{L^2(\mu)}^2
\le
4(\eta_\epsilon^\star+\delta_{\mathrm{or}})^2$. Defining $\Delta_N:=f_{\widehat\phi_N}-G_{r_{\delta_{\mathrm{or}}}}$ we get
Then $\|\Delta_N\|_{L^2(\mu)}^2
\le
2\|f_{\widehat\phi_N}-G^{\pi_B}\|_{L^2(\mu)}^2
+
2\|G_{r_{\delta_{\mathrm{or}}}}-G^{\pi_B}\|_{L^2(\mu)}^2$.
Substituting the two previous bounds yields, with probability at least $1-\delta$,
\[
\|\Delta_N\|_{L^2(\mu)}^2
\le
4(2\eta_\epsilon^\star+\eta_{\mathrm{cls}}+\eta_{\mathrm{traj}})^2
+
4\Gamma_N(\delta)
+
4\eta_{\mathrm{traj}}^2
+
8(\eta_\epsilon^\star+\delta_{\mathrm{or}})^2.
\]

For each $(s,u)\in \operatorname{supp}(\lambda)$, define $d^N_{s,u}:=\widehat\phi_N(u)-\widehat\phi_N(s)\in D_N$ and $d^{r_{\delta_{\mathrm{or}}}}_{s,u}
:=
\phi^{r_{\delta_{\mathrm{or}}}}(u)-\phi^{r_{\delta_{\mathrm{or}}}}(s)
\in D_{r_{\delta_{\mathrm{or}}}}$.
By construction we get $(A_N d^N_{s,u})(g)=f_{\widehat\phi_N}(s,u,g)$ and $(B_{r_{\delta_{\mathrm{or}}}} d^{r_{\delta_{\mathrm{or}}}}_{s,u})(g)
=
G_{r_{\delta_{\mathrm{or}}}}(s,u,g)$. Hence $(A_N d^N_{s,u})(g)-(B_{r_{\delta_{\mathrm{or}}}}d^{r_{\delta_{\mathrm{or}}}}_{s,u})(g)
=
\Delta_N(s,u,g)$. Integrating  against first $\nu$ and then $\lambda$ gives
\[
\int_{\mathcal X\times \mathcal X}
\|A_N d^N_{s,u}-B_{r_{\delta_{\mathrm{or}}}}d^{r_{\delta_{\mathrm{or}}}}_{s,u}\|_{L^2(\mathcal X,\nu)}^2
\,d\lambda(s,u)
=
\int_{\mathcal X\times \mathcal X}\int_{\mathcal X}
|\Delta_N(s,u,g)|^2\,d\nu(g)\,d\lambda(s,u).
\]

By the coverage assumption we get
$\nu_{s,u}(B)\ge c_{\mathrm{cov}}\,\nu(B)$ for $\lambda\text{-a.e. }(s,u)$,
and every measurable $B\subseteq \mathcal X$. Therefore
\begin{align*}
c_{\mathrm{cov}} 
\int_{\mathcal X\times \mathcal X}\int_{\mathcal X} &
|\Delta_N(s,u,g)|^2\,d\nu(g)\,d\lambda(s,u)\\
&\le
\int_{\mathcal X\times \mathcal X}\int_{\mathcal X}
|\Delta_N(s,u,g)|^2\,d\nu_{s,u}(g)\,d\lambda(s,u)
=
\|\Delta_N\|_{L^2(\mu)}^2.    
\end{align*}

Thus $\int_{\mathcal X\times \mathcal X}
\|A_N d^N_{s,u}-B_{r_{\delta_{\mathrm{or}}}}d^{r_{\delta_{\mathrm{or}}}}_{s,u}\|_{L^2(\mathcal X,\nu)}^2
\,d\lambda(s,u)
\le
\frac{1}{c_{\mathrm{cov}}}\|\Delta_N\|_{L^2(\mu)}^2$.
Since $B_{r_{\delta_{\mathrm{or}}}}$ is bounded below, it is injective, has closed range, and has a bounded inverse on its range. Let $m_{B,\delta_{\mathrm{or}}}>0$ denote its lower-bound constant, so that $\|B_{r_{\delta_{\mathrm{or}}}}d\|_{L^2(\mathcal X,\nu)}
\ge
m_{B,\delta_{\mathrm{or}}}\|d\|_{\mathcal H^{r_{\delta_{\mathrm{or}}}}}$ with $d\in D_{r_{\delta_{\mathrm{or}}}}$. Equivalently,
$\|B_{r_{\delta_{\mathrm{or}}}}^{-1}f\|_{\mathcal H^{r_{\delta_{\mathrm{or}}}}}
\le
\frac{1}{m_{B,\delta_{\mathrm{or}}}}
\|f\|_{L^2(\mathcal X,\nu)}$ with $f\in\mathcal R(B_{r_{\delta_{\mathrm{or}}}})$. Let $P_{\mathcal R(B_{r_{\delta_{\mathrm{or}}}})}$
denote the orthogonal projection in $L^2(\mathcal X,\nu)$ onto
$\mathcal R(B_{r_{\delta_{\mathrm{or}}}})$, and define $T_{N,\delta_{\mathrm{or}}}
:=
B_{r_{\delta_{\mathrm{or}}}}^{-1}
P_{\mathcal R(B_{r_{\delta_{\mathrm{or}}}})}
A_N$. Then $T_{N,\delta_{\mathrm{or}}}:D_N\to D_{r_{\delta_{\mathrm{or}}}}$
is a bounded linear map. Since $B_{r_{\delta_{\mathrm{or}}}}d^{r_{\delta_{\mathrm{or}}}}_{s,u}
\in
\mathcal R(B_{r_{\delta_{\mathrm{or}}}})$,
we have $P_{\mathcal R(B_{r_{\delta_{\mathrm{or}}}})}
B_{r_{\delta_{\mathrm{or}}}}d^{r_{\delta_{\mathrm{or}}}}_{s,u}
=
B_{r_{\delta_{\mathrm{or}}}}d^{r_{\delta_{\mathrm{or}}}}_{s,u}$.
Therefore,
\[
\begin{aligned}
\left\|
T_{N,\delta_{\mathrm{or}}}d^N_{s,u}
-
d^{r_{\delta_{\mathrm{or}}}}_{s,u}
\right\|_{\mathcal H^{r_{\delta_{\mathrm{or}}}}}
&=
\left\|
B_{r_{\delta_{\mathrm{or}}}}^{-1}
P_{\mathcal R(B_{r_{\delta_{\mathrm{or}}}})}
\left(
A_Nd^N_{s,u}
-
B_{r_{\delta_{\mathrm{or}}}}d^{r_{\delta_{\mathrm{or}}}}_{s,u}
\right)
\right\|_{\mathcal H^{r_{\delta_{\mathrm{or}}}}}
\\
&\le
\frac{1}{m_{B,\delta_{\mathrm{or}}}}
\left\|
A_Nd^N_{s,u}
-
B_{r_{\delta_{\mathrm{or}}}}d^{r_{\delta_{\mathrm{or}}}}_{s,u}
\right\|_{L^2(\mathcal X,\nu)}.
\end{aligned}
\]
Squaring and integrating over $(s,u)\sim\lambda$ gives
\[
\begin{aligned}
\int_{\mathcal X\times\mathcal X}
\left\|
T_{N,\delta_{\mathrm{or}}}d^N_{s,u}
-
d^{r_{\delta_{\mathrm{or}}}}_{s,u}
\right\|_{\mathcal H^{r_{\delta_{\mathrm{or}}}}}^{2}
\,d\lambda(s,u)
\le
\frac{1}{m_{B,\delta_{\mathrm{or}}}^{2}}
\int_{\mathcal X\times\mathcal X}
\left\|
A_Nd^N_{s,u}
-
B_{r_{\delta_{\mathrm{or}}}}d^{r_{\delta_{\mathrm{or}}}}_{s,u}
\right\|_{L^2(\mathcal X,\nu)}^{2}
\,d\lambda(s,u).
\end{aligned}
\]
Using the preceding bound on the analysis-operator discrepancy, we obtain
\begin{align*}
&\int_{\mathcal X\times\mathcal X}
\left\|
T_{N,\delta_{\mathrm{or}}}d^N_{s,u}
-
d^{r_{\delta_{\mathrm{or}}}}_{s,u}
\right\|_{\mathcal H^{r_{\delta_{\mathrm{or}}}}}^{2}
\,d\lambda(s,u)
\\
&\qquad \qquad \qquad \le
\frac{4}{c_{\mathrm{cov}}\,m_{B,\delta_{\mathrm{or}}}^{2}}
\Big[
(2\eta_\epsilon^\star+\eta_{\mathrm{cls}}+\eta_{\mathrm{traj}})^2
+
\Gamma_N(\delta)
+
\eta_{\mathrm{traj}}^{2}
+
2(\eta_\epsilon^\star+\delta_{\mathrm{or}})^2
\Big].
\end{align*}
Finally, substituting $d^N_{s,u}
=
\widehat\phi_N(u)-\widehat\phi_N(s)$ and $d^{r_{\delta_{\mathrm{or}}}}_{s,u}
=
\phi^{r_{\delta_{\mathrm{or}}}}(u)-\phi^{r_{\delta_{\mathrm{or}}}}(s)$
yields
\[
\begin{aligned}
&\int_{\mathcal X\times\mathcal X}
\bigl\|
T_{N,\delta_{\mathrm{or}}}
(\widehat\phi_N(u)-\widehat\phi_N(s))
-
(\phi^{r_{\delta_{\mathrm{or}}}}(u)-\phi^{r_{\delta_{\mathrm{or}}}}(s))
\bigr\|_{\mathcal H^{r_{\delta_{\mathrm{or}}}}}^{2}
\,d\lambda(s,u)
\\
&\qquad \qquad \qquad \le
\frac{4}{c_{\mathrm{cov}}\,m_{B,\delta_{\mathrm{or}}}^{2}}
\Big[
(2\eta_\epsilon^\star+\eta_{\mathrm{cls}}+\eta_{\mathrm{traj}})^2
+
\Gamma_N(\delta)
+
\eta_{\mathrm{traj}}^{2}
+
2(\eta_\epsilon^\star+\delta_{\mathrm{or}})^2
\Big],
\end{aligned}
\]
\end{proof}

%% file: unsupervised-policy-training.tex
\section{On the supervision of policy training in offline goal-conditioned RL}
\label{app:goal_agnostic_policy_training}

Table~\ref{tab:methods-comparison} distinguishes between ordinary reward-free or zero-shot training and a stricter property denoted as \emph{Unsupervised Policy Training}. This column should be interpreted as goal-agnostic training and not as a claim that other methods such as FB are not unsupervised RL.  This distinction is important because many methods are unsupervised in the sense that they do not use
externally provided reward labels during pretraining, but nevertheless train their policy or actor on conditioning variables that are constructed from dataset states treated as goals, targets, or reward/task latents.

We use the following operational definition. Consider a latent-conditioned policy $\pi(a\mid x,z)$. We say that its training is
\emph{goal-agnostic} if, during policy optimization, the conditioning variable $z$ is sampled from an exogenous distribution over abstract latent directions and is not instantiated as an embedding of a dataset state selected as a goal.  We allow the representation or geometry-learning stage to use states, future states, temporal gaps, or
self-supervised relabeling. Our criterion concerns the \emph{policy-training stage}: whether the policy network itself is trained by seeing real goals, directly as goal inputs or indirectly through goal-derived latent variables.

Under this definition, HILP satisfies goal-agnostic policy training
in its core formulation. HILP first learns a temporal representation
$\phi$, and then trains a latent-conditioned policy over arbitrary directions $z$ in the learned representation space using the intrinsic reward $r_z(x,x')=\langle \phi(x')-\phi(x), z\rangle.$ The goal state is not used as the policy-conditioning variable during this
policy-training phase. A downstream goal is converted into a direction only at test time, for example by using the direction from the current state embedding toward the goal embedding.

IEL follows the same separation principle but replaces HILP's symmetric
geometry with a directed hitting-time geometry. In IEL, replay-buffer states and intermediate trajectory states are used to learn the embedding $\phi$ and the task identifier $\omega_\psi(g)$ through self-supervised temporal-gap and identifier losses. This use of goals is part of representation learning, not policy supervision. The resulting foundation policy is trained over abstract directions in the learned geometry, and concrete goals enter only at
test time through prompting or graph-based planning. Thus, the policy is not trained on goal labels even though the representation is learned from self-supervised goal relabeling.

This definition separates HILP and our IEL from the rest of  reward-free or zero-shot methods. FB representations are reward-free and zero-shot in the usual sense. They learn from reward-free interactions and can construct policies for rewards specified after pretraining. However, the policy parameter in FB is a reward/task latent. Once a reward is specified, FB forms $ z_R = \mathbb E[r(s,a)B(s,a)]$ and for a target-state reward this latent can be instantiated from the backward representation of the target state. Moreover, standard deep zero-shot FB implementations sample task vectors partly by passing replay-buffer states through the backward representation, i.e. $z=B(s)$ for sampled states $s$. Thus, FB is reward-free but its policy-training distribution is at least partly anchored to reward or goal-state latents rather than purely to arbitrary directions in a learned temporal geometry.

Similarly, goal-conditioned RL and reachability-learning methods train
policies, critics, or classifiers on state--goal pairs sampled from the replay buffer, often using future-state or random-state relabeling. These methods may be self-supervised or reward-free, but their policy updates explicitly use dataset states as goals. Therefore they do not satisfy the stricter goal-agnostic policy criterion used in Table~\ref{tab:methods-comparison}. The last column of the table identifies whether the policy is trained as a reusable foundation policy whose conditioning variables are not real goals or goal-derived task latents. This is the property needed for our multi-stage planning setting, where concrete goals should be introduced only at test time and composed through the learned temporal geometry.

%% file: algorithms/IEL.tex
\section{Algorithm Pseudocode}
\label{sec:algorithm_pseudocode}

\begin{algorithm}
\caption{Isomorphic Embedding Learning} 
\label{alg:IEL}
\begin{algorithmic}[1]
\State \textbf{Input:} Replay buffer $\mathcal{B}$, value expectile $\tau$, hitting time expectile $\tau'$, topological consistency temperature $\beta$,  topological consistency coefficient $\kappa$, discount factor $\gamma$

\State \textbf{Initialize:} Encoder $\phi_\theta$ with parameters $\bar{\theta} \gets \theta$

\For{each training iteration}
    \State Sample batch $(x_i, u_i, H(u_i), x'_i, g_i) \sim \mathcal{B}$ \Comment{$u_i$: intermediate state with hitting time $H(u_i)$}
    
    \Statex \hrulefill \ \textit{Temporal Difference Learning} \hrulefill

    \State $\Delta'_{g_i}  \gets \phi_\theta(g_i) - \phi_\theta(x'_i)$

    \State $\cos(\xi'_i)=\langle \Delta'_{g_i}, \omega_\psi(g_i)  \rangle / (\|  \Delta'_{g_i} \| \cdot \| \omega_\psi(g_i)\|)$

    \State $\Delta_{g_i}  \gets \phi_\theta(g_i) - \phi_\theta(x_i)$   
    
    \State $\cos(\xi_i)=\langle \Delta_{g_i}, \omega_\psi(g_i)  \rangle / (\|  \Delta_{g_i} \| \cdot \| \omega_\psi(g_i)\|)$
    
    \State $\mathcal{T}_i \gets \mathbb{I}(x_i=g_i) \cdot 0 + \mathbb{I}(x_i \neq g_i) \cdot (1 + \gamma \| \Delta'_{g_i}  \| \cdot \exp(\beta (1- \cos(\xi'_i))))$ \Comment{Bellman target}

    \State $\delta_i \gets \lfloor \mathcal{T}_i \rfloor_{sg} - \| \Delta_{g_i}  \| \cdot \exp(\beta (1- \cos(\xi_i))) $ \Comment{Bellman residual}

    \State $w_i \gets |\tau - \mathbb{I}(\delta_i < 0)|$   \Comment{Expectile weights for value learning}

    \Statex \hrulefill \ \textit{Hitting Time Regression} \hrulefill

    \State $\Delta_{u_i}\gets \phi_\theta(u_i) - \phi_\theta(x_i)$

    \State $\ell_i \gets \frac{1-\gamma^{H(u_i)}}{1- \gamma} - \langle \Delta_{u_i}, \omega_\psi(g_i) / \| \omega_\psi(g_i) \| \rangle$ \Comment{Discounted hitting time residual}

    \State $w'_i \gets |\tau' - \mathbb{I}(\ell_i < 0)|$

     \Statex \hrulefill \ \textit{Optimization} \hrulefill

    \State $\mathcal{L}_{emb} \gets \frac{1}{N} \sum_{i=1}^{N} \left( w_i \delta_i^2 + \kappa w'_i \ell_i^2 \right )$

    \State $\theta \gets \theta - \alpha \nabla_\theta \mathcal{L}_{emb}$ \Comment{Gradient update}
\EndFor
\end{algorithmic}
\end{algorithm}

%% file: algorithms/InfoNCE.tex
\begin{algorithm}
\caption{Continuous Contrastive Task Learning (InfoNCE like)}
\label{alg:InfoNCE}
\begin{algorithmic}[1]
\State \textbf{Input:} Replay buffer $\mathcal{B}$, task encoder $\omega_\psi$, temperature $\tau$, augmentation scale $\sigma$

\State \textbf{Initialize:} Encoder parameters $\psi$

\For{each training iteration}
    \State Sample batch $(g_i, u_i, x_i) \sim \mathcal{B}$ \Comment{$g_i$: goals, $u_i$: intermediate, $x_i$: states}
    
    \State $\mathbf{x} \gets [g_i; u_i; x_i]$ \Comment{Concatenate all target states into a single batch}
    
    \State $\mathbf{x}' \gets \mathbf{x} + \epsilon, \quad \epsilon \sim \mathcal{N}(0, \sigma^2 \mathbf{I})$ \Comment{Apply Gaussian noise augmentation}

    \State $z \gets \omega_\psi(\mathbf{x})$ \Comment{Embed original states}
    \State $z' \gets \omega_\psi(\mathbf{x}')$ \Comment{Embed augmented states}

    \State $\hat{z} \gets z / \|z\|, \quad \hat{z}' \gets z' / \|z'\|$ \Comment{L2 normalization to the unit hypersphere}

    \State $\mathcal{S} \gets \langle \hat{z}, (\hat{z}')^\top \rangle / \tau$ \Comment{Compute $N \times N$ similarity matrix}

    \State $\mathcal{L} \gets \frac{1}{N} \sum_{i=1}^{N} -\log \left( \frac{\exp(\mathcal{S}_{i,i})}{\sum_{j=1}^{N} \exp(\mathcal{S}_{i,j})} \right)$ \Comment{Compute InfoNCE / NT-Xent loss}
    
    \State Update $\omega_\psi$ by minimizing $\mathcal{L}$
\EndFor
\end{algorithmic}
\end{algorithm}

%% file: algorithms/three_phase_training.tex
\begin{algorithm}[H]
\caption{Three phase training}
\label{alg:three_phase_training}
\begin{algorithmic}[1]

\State \textbf{Input:} Replay buffer $\mathcal{B}$, task encoder $\omega_\psi$, state encoder $\phi_\theta$, policy $\pi(a|x,z)$ 

\State \textbf{Initialize:} Encoder parameters $\theta, \psi, \pi$

\While{not converged} \Comment{Phase 1: Task encoder}
    \State train $\omega_\psi$ using algorithm \ref{alg:InfoNCE} 
\EndWhile

\State Freeze task encoder parameters $\psi$ 
\While{not converged} \Comment{Phase 2: State encoder}
    \State train $\phi_\theta(x)$ using algorithm \ref{alg:IEL}
\EndWhile

\State Freeze state encoder parameters $\theta$
\While{not converged} \Comment{Phase 3: Policy learning}
    \State train $\pi(a|x,z)$ using any off-the-shelf offline RL algo with uniformly random unit-length $z$ 
\EndWhile
\end{algorithmic}
\end{algorithm}

%% file: algorithms/asymmetric_graph_planning.tex
\begin{algorithm}
\caption{Asymmetric Graph Planning}
\label{alg:asymmetric_graph_planning}
\begin{algorithmic}[1]
\State \textbf{Input:} Coreset $\mathcal{X}_c$, Goal $g$, Value Function $V$, Encoder $\phi$, Neighbors $k$

\Statex \hrulefill \ \textit{Graph Construction} \hrulefill

\Function{ConstructGraph}{$\mathcal{X}_c, g, k$}
    \State $\mathbf{C}_{i,j} \gets d(x_i, x_j, g) \quad \forall x_i, x_j \in \mathcal{X}_c$ \Comment{Asymmetric cost matrix}
    \State $\mathbf{C}_{sym} \gets (\mathbf{C} + \mathbf{C}^\top) / 2$ \Comment{Symmetrize for connectivity backbone}
    \State $\mathcal{T} \gets \text{MinimumSpanningTree}(\mathbf{C}_{sym})$
    
    \State $\mathbf{A} \gets \text{KeepTopKPerRow}(\mathbf{C}, k)$ \Comment{Initial sparse directed adjacency}
    
    \For{each edge $(i, j) \in \mathcal{T}$} \Comment{Force MST edges to ensure global connectivity}
        \State $\mathbf{A}_{i,j} \gets \mathbf{C}_{i,j}$ 
        \State $\mathbf{A}_{j,i} \gets \mathbf{C}_{j,i}$
    \EndFor
    
    \State \Return $\mathbf{A}^\top$ \Comment{Return reversed graph for SSSP from goal}
\EndFunction

\Statex \hrulefill \ \textit{Online Planning Policy} \hrulefill

\Function{PlanStep}{$\mathcal{X}_c, x, g, \mathbf{A}_{rev}$}
    \State $v_{goal} \gets \arg\min_{i} d(x_i, g, g)$ \Comment{Match coreset node to goal}
    \State $\mathcal{P} \gets \text{ShortestPathPredecessors}(\mathbf{A}_{rev}, \text{source}=v_{goal})$
    
    \State $v_{curr} \gets \arg\min_{i} d(x, x_i, g) \quad \forall x_i \in \mathcal{X}_c$ \Comment{Localize current state}
    \State $v_{next} \gets \mathcal{P}[v_{curr}]$
    
    \If{$v_{next} \neq \text{null}$ \textbf{and} $v_{curr} \neq v_{goal}$}
        \State $\phi_{target} \gets \phi(v_{next})$
    \Else
        \State $\phi_{target} \gets \phi(g)$ \Comment{Direct pursuit if no path exists}
    \EndIf
    
    \State $z \gets (\phi_{target} - \phi(x))/\|\phi_{target} - \phi(x)\|$ \Comment{Task identifier}
    \State \Return $\pi(x, z)$
\EndFunction
\end{algorithmic}
\end{algorithm}

%% file: tables/main_configuration.tex
\section{Experiment details and additional results}
\label{sec:experiment_details}

{\bf Hyperparameter configuration.} See Table \ref{table:main_configuration} for the hyperparameter configurations used to generate the results in Table \ref{table:main_results}. Several parameters specific to Isomorphic Embedding Learning (IEL) can be tuned to optimize performance based on the environment's structural characteristics and the quality of the offline dataset.

{\bf Hyperparameter sensitivity.} As shown in Table \ref{tab:iel_comparison}, almost all evaluated configurations yield performance within two standard deviations with respect to the best-performing configuration. This proximity suggests a failure to reject the null hypothesis—that there is no statistically significant difference between these configurations—under standard significance testing. We therefore conclude that IEL is robust and not overly sensitive to the hyperparameter ranges.

\textbf{Tuning rules of thumb.} Firstly, $\beta$ determines the \emph{sharpness} of the enforced topology in the embedding space, defining the perceived distance between nearby points. For datasets with high noise or low structure, a low $\beta$ ($< 0.1$) is recommended for training stability. In contrast, datasets with high-quality, expert trajectories can be better resolved with $\beta \in [0.3, 1.0]$.  Secondly, $H_{\max}$ defines the prediction horizon for hitting times. While low values ($H_{\max} \le 5$) ensure stability, they limit the agent's ability to differentiate between long-range goals. Conversely, excessively large values ($H_{\max} \ge 30$) can introduce variance that leads to chaotic predictions. For general use, $H_{\max}$ is relatively easy to tune by observing the distribution of trajectory lengths; it should be set high enough to capture meaningful progress but low enough to avoid the noise of unstructured exploration. Thirdly, $\tau'$ denotes the expectile for hitting time regression. Decreasing $\tau'$ encourages the exploitation of shorter paths, which is effective when the hitting time range $H(u)$ is narrow (e.g., in expert datasets). When dataset quality is unknown, the unbiased $\tau'=0.5$ provides the most robust results.

\textbf{Coreset selection and computational efficiency.} Regarding evaluation, the coreset selection parameter $\sigma$ determines the similarity threshold between points. This parameter is heavily dependent on the embedding scale set by $H_{\max}$. We find that $\sigma \in [0.3\mu, 0.75\mu]$, where $\mu$ is the mean pairwise distance in the coreset, provides a reliable heuristic for tuning. Coreset selection primarily serves two distinct purposes: it limits the number of vertices in the planning graph to control computational complexity, and it removes distribution bias from the dataset. By reflecting the spatial coverage rather than the density of the original data, the coreset enables unbiased, coarse-grained long-horizon planning. Increasing the size of the coreset beyond the point of sufficient spatial coverage yields diminishing returns and can potentially degrade planning performance. This occurs when the graph becomes too fine-grained, presenting short-distance targets that fall below the resolution of the hitting-time predictions the agent was trained on. Consequently, an ideal coreset is one where the mean distance to $k$-nearest neighbor ($k$-NN) vertices is approximately constant across the state space and remains proportional to the hitting times seen during training. We configured the coreset size to ensure \emph{Asym-Graph} evaluation remains proportional to \emph{Rec-Mid} wall-clock time. Specifically, \emph{Rec-Mid} with 50,000 samples requires 400–600 seconds, while \emph{Asym-Graph} with a coreset size of $2^{13}=8192$ takes approximately 700 seconds. For reference‚ \emph{Sym-Graph} requires only $150$ seconds. These timings vary with the specific graph topology and complexity.

All experiments were conducted on a system equipped with an NVIDIA GeForce RTX 4090 (24 GB VRAM) and an Intel Core i7-14700K (96 GB RAM). A full experimental cycle, including model training and final evaluation, requires approximately one hour.

\begin{table}[t!]
\centering
\small
\setlength{\tabcolsep}{3pt}
\caption{Main configuration of IEL hyperparameters. Only hyperparameters specific to IEL or Asymmetric/Symmetric Graph Planning differ from those presented in HILP \cite{park2024hilp}.}
\label{table:main_configuration}
\begin{tabular*}{\textwidth}{@{\extracolsep{\fill}}lc}
\toprule
Hyperparameter & Value \\
\midrule
\# total gradient steps  & $10^6$  \\
\# phase 0 steps (task learning) & $0.2\cdot 10^5$ \\
\# phase 1 steps (embedding learning) & $4.8\cdot 10^5$ \\
\# phase 2 steps (policy learning) & $5\cdot 10^5$ \\
Learning rate  &  $3\cdot 10^{-4}$ \\
Optimizer & Adam \\
Minibatch-size & 1024 \\
Actor and value MLP hidden layer dimensions & $(512, 512, 512)$ \\
Discount factor $\gamma$ & $0.99$ \\
Latent embedding ($\phi$) dimension & 32 \\
Value learning expectile $\tau$ & $0.95$ \\
Hitting time regression expectile $\tau'$ & $0.5$ \\
Topological consistency temperature $\beta$ & $0.1$ \\
Maximum in-trajectory intermediate state hitting time $H_{\max}$ & $10$ \\
Target smoothing coefficient & $0.005$ \\
Actor temperature & $10$ \\
Evaluation coreset size & $8192$ \\
Evaluation coreset selection $\sigma$ & $20$ \\
Evaluation graph neighbor count ($k$ for $k$-NN) & $10$\\
\bottomrule
\end{tabular*}
\end{table}

%% file: tables/many_configurations_results.tex
\begin{table}[h!]
\centering
\small
\setlength{\tabcolsep}{5pt}
\caption{Comparison of IEL Asymmetric and Symmetric Graph Planning configurations across environments.}
\resizebox{\textwidth}{!}{
\begin{tabular}{lcccccc}
\toprule
Configuration & \rotatebox{45}{antmaze-large-diverse} & \rotatebox{45}{antmaze-large-play} & \rotatebox{45}{antmaze-ultra-diverse} & \rotatebox{45}{antmaze-ultra-play} & \rotatebox{45}{kitchen-partial} & \rotatebox{45}{kitchen-mixed} \\
\midrule
\multicolumn{7}{l}{\textit{Asymmetric Planning Performance}} \\
\midrule
$H_{\max}=10.0., \sigma=0.25$ & 58.0 \std{18.8} & 51.5 \std{16.5} & 36.5 \std{32.9} & 45.5 \std{23.9} & 53.2 \std{4.0} & 50.4 \std{6.3} \\
$H_{\max}=20.0., \sigma=0.25$ & \textbf{73.7 \std{13.4}} & 53.8 \std{16.1} & 35.0 \std{24.8} & 28.5 \std{23.2} & \textbf{61.4 \std{5.7}} & 54.7 \std{5.2} \\
$H_{\max}=10.0., \sigma=5.0$ & 57.0 \std{10.3} & 52.2 \std{16.4} & 61.2 \std{15.7} & 61.0 \std{14.1} & 54.8 \std{3.7} & 53.4 \std{6.2} \\
$H_{\max}=20.0., \sigma=5.0$ & 67.4 \std{8.2} & 40.5 \std{16.9} & 68.8 \std{15.6} & 55.0 \std{16.5} & 59.9 \std{5.2} & 51.5 \std{9.4} \\
$H_{\max}=10.0., \sigma=10.0$ & 61.8 \std{9.3} & 58.0 \std{10.9} & 66.8 \std{8.7} & 70.0 \std{10.6} & 56.0 \std{5.9} & 50.6 \std{7.2} \\
$H_{\max}=20.0., \sigma=10.0$ & 65.4 \std{9.0} & 45.8 \std{17.0} & 59.0 \std{10.8} & 66.0 \std{16.8} & 58.8 \std{8.6} & 51.4 \std{8.6} \\
$H_{\max}=10.0., \sigma=20.0$ & 71.8 \std{6.5} & \textbf{62.8 \std{5.9}} & \textbf{79.0 \std{4.5}} & 73.8 \std{6.2} & 60.4 \std{3.5} & 57.2 \std{3.8} \\
$H_{\max}=20.0., \sigma=20.0$ & 73.2 \std{6.7} & 55.8 \std{12.1} & 74.2 \std{6.0} & \textbf{75.2 \std{3.2}} & 59.0 \std{2.3} & \textbf{58.0 \std{3.9}} \\
$H_{\max}=10.0., \sigma=50.0$ & 44.5 \std{21.6} & 51.8 \std{20.2} & 51.8 \std{21.7} & 71.2 \std{3.2} & 61.1 \std{2.6} & 57.3 \std{3.9} \\
$H_{\max}=20.0., \sigma=50.0$ & 53.1 \std{24.3} & 46.0 \std{24.8} & 61.5 \std{7.2} & 63.5 \std{10.5} & 60.4 \std{1.8} & 57.2 \std{2.0} \\
\midrule
\addlinespace
\multicolumn{7}{l}{\textit{Symmetric Planning Performance}} \\
\midrule
$H_{\max}=10.0., \sigma=0.25$ & 59.0 \std{14.2} & 60.5 \std{6.2} & 69.5 \std{17.4} & 57.5 \std{32.0} & 55.8 \std{7.0} & 51.3 \std{6.6} \\
$H_{\max}=20.0., \sigma=0.25$ & \textbf{73.4 \std{12.2}} & 53.2 \std{22.4} & \textbf{74.8 \std{9.4}} & 65.0 \std{28.0} & 59.3 \std{4.0} & 54.1 \std{5.0} \\
$H_{\max}=10.0., \sigma=5.0$ & 61.5 \std{20.0} & 55.5 \std{10.5} & 63.0 \std{12.9} & 62.8 \std{12.5} & 57.8 \std{5.0} & 50.5 \std{6.9} \\
$H_{\max}=20.0., \sigma=5.0$ & 66.3 \std{10.3} & 41.5 \std{20.6} & 65.5 \std{12.7} & 58.8 \std{20.6} & 60.4 \std{4.1} & 52.2 \std{9.5} \\
$H_{\max}=10.0., \sigma=10.0$ & 57.0 \std{18.9} & 55.0 \std{8.3} & 59.0 \std{19.5} & 64.5 \std{11.3} & 56.7 \std{6.8} & 52.1 \std{7.2} \\
$H_{\max}=20.0., \sigma=10.0$ & 63.1 \std{7.6} & 40.0 \std{21.9} & 64.2 \std{10.9} & \textbf{65.2 \std{23.8}} & 61.0 \std{3.5} & 53.4 \std{8.2} \\
$H_{\max}=10.0., \sigma=20.0$ & 55.0 \std{20.6} & \textbf{63.0 \std{14.9}} & 62.5 \std{18.7} & \textbf{65.2 \std{13.5}} & 60.4 \std{5.5} & 57.8 \std{5.1} \\
$H_{\max}=20.0., \sigma=20.0$ & 72.5 \std{10.0} & 46.5 \std{11.4} & 58.8 \std{11.6} & 60.0 \std{19.2} & 58.2 \std{7.0} & \textbf{58.5 \std{3.8}} \\
$H_{\max}=10.0., \sigma=50.0$ & 57.2 \std{12.0} & 52.8 \std{10.7} & 47.8 \std{18.2} & 56.8 \std{24.7} & 60.9 \std{4.5} & 54.1 \std{4.4} \\
$H_{\max}=20.0., \sigma=50.0$ & 55.4 \std{23.0} & 38.2 \std{26.0} & 62.5 \std{10.5} & 52.5 \std{16.7} & \textbf{61.2 \std{3.9}} & 56.4 \std{3.2} \\
\bottomrule
\end{tabular}
}
\label{tab:iel_comparison}
\end{table}